\documentclass{article}

\usepackage{microtype}
\usepackage{graphicx}
\usepackage{subfigure}
\usepackage{booktabs} 

\usepackage{hyperref}

\usepackage[accepted]{icml2020}
\usepackage{multirow}
\usepackage{multicol}
\usepackage{amssymb}
\usepackage{amsmath}
\usepackage{amsfonts}
\usepackage{amsthm}
\usepackage{subfigure}
\usepackage{graphicx}
\usepackage{multirow}
\usepackage{paralist}
\usepackage{mathrsfs, euscript}
\usepackage{color}
\usepackage{xcolor}
\usepackage{nccmath}
\usepackage{bm}

\icmltitlerunning{Continuously Indexed Domain Adaptation}

\newcommand{\tabref}[1]{Table~\ref{#1}}
\newcommand{\secref}[1]{Sec.~\ref{#1}}
\newcommand{\figref}[1]{Fig.~\ref{#1}}
\newcommand{\lemref}[1]{Lemma~\ref{#1}}
\newcommand{\thmref}[1]{Theorem~\ref{#1}}

\newcommand{\crlref}[1]{Corollary~\ref{#1}}
\newcommand{\eqnref}[1]{Eqn.~\ref{#1}}

\newenvironment{Itemize}%
{\begin{itemize}%
\setlength{\itemsep}{0pt}%
\setlength{\topsep}{0pt}%
\setlength{\partopsep}{0pt}%
\setlength{\parskip}{0pt}}%
{\end{itemize}}
\begin{document}

\twocolumn[
\icmltitle{Continuously Indexed Domain Adaptation}



\icmlsetsymbol{equal}{*}

\begin{icmlauthorlist}
\icmlauthor{Hao Wang}{equal,mit}
\icmlauthor{Hao He}{equal,mit}
\icmlauthor{Dina Katabi}{mit}
\end{icmlauthorlist}

\icmlaffiliation{mit}{MIT Computer Science and Artificial Intelligence Laboratory, Massachusetts, USA}

\icmlcorrespondingauthor{Hao Wang}{hoguewang@gmail.com}

\icmlkeywords{Domain Adaptation, Deep Learning, Machine Learning}

\vskip 0.3in
]



\printAffiliationsAndNotice{\icmlEqualContribution} 

\def\Blue{\color{blue}}
\def\Purple{\color{purple}}

\def\A{{\bf A}}
\def\a{{\bf a}}
\def\B{{\bf B}}
\def\b{{\bf b}}
\def\C{{\bf C}}
\def\c{{\bf c}}
\def\D{{\bf D}}
\def\d{{\bf d}}
\def\E{{\bf E}}
\def\e{{\bf e}}
\def\f{{\bf f}}
\def\F{{\bf F}}
\def\K{{\bf K}}
\def\k{{\bf k}}
\def\L{{\bf L}}
\def\H{{\bf H}}
\def\h{{\bf h}}
\def\G{{\bf G}}
\def\g{{\bf g}}
\def\I{{\bf I}}
\def\J{{\bf J}}
\def\R{{\bf R}}
\def\X{{\bf X}}
\def\Y{{\bf Y}}
\def\OO{{\bf O}}
\def\oo{{\bf o}}
\def\P{{\bf P}}
\def\p{{\bf p}}
\def\Q{{\bf Q}}
\def\q{{\bf q}}
\def\r{{\bf r}}
\def\s{{\bf s}}
\def\S{{\bf S}}
\def\t{{\bf t}}
\def\T{{\bf T}}
\def\x{{\bf x}}
\def\y{{\bf y}}
\def\z{{\bf z}}
\def\Z{{\bf Z}}
\def\M{{\bf M}}
\def\m{{\bf m}}
\def\n{{\bf n}}
\def\U{{\bf U}}
\def\u{{\bf u}}
\def\V{{\bf V}}
\def\v{{\bf v}}
\def\W{{\bf W}}
\def\w{{\bf w}}
\def\0{{\bf 0}}
\def\1{{\bf 1}}

\def\AM{{\mathcal A}}
\def\EM{{\mathcal E}}
\def\FM{{\mathcal F}}
\def\TM{{\mathcal T}}
\def\UM{{\mathcal U}}
\def\XM{{\mathcal X}}
\def\YM{{\mathcal Y}}
\def\NM{{\mathcal N}}
\def\OM{{\mathcal O}}
\def\IM{{\mathcal I}}
\def\GM{{\mathcal G}}
\def\PM{{\mathcal P}}
\def\LM{{\mathcal L}}
\def\MM{{\mathcal M}}
\def\DM{{\mathcal D}}
\def\SM{{\mathcal S}}
\def\ZM{{\mathcal Z}}
\def\RB{{\mathbb R}}
\def\EB{{\mathbb E}}
\def\VB{{\mathbb V}}

\def\tx{\tilde{\bf x}}
\def\ty{\tilde{\bf y}}
\def\tz{\tilde{\bf z}}
\def\hd{\hat{d}}
\def\HD{\hat{\bf D}}
\def\hx{\hat{\bf x}}
\def\hR{\hat{R}}

\def\Ome{\mbox{\boldmath$\omega$\unboldmath}}
\def\Om{\mbox{\boldmath$\Omega$\unboldmath}}
\def\bet{\mbox{\boldmath$\beta$\unboldmath}}
\def\et{\mbox{\boldmath$\eta$\unboldmath}}
\def\ep{\mbox{\boldmath$\epsilon$\unboldmath}}
\def\ph{\mbox{\boldmath$\phi$\unboldmath}}
\def\Pii{\mbox{\boldmath$\Pi$\unboldmath}}
\def\pii{\mbox{\boldmath$\pi$\unboldmath}}
\def\Ph{\mbox{\boldmath$\Phi$\unboldmath}}
\def\Ps{\mbox{\boldmath$\Psi$\unboldmath}}
\def\tha{\mbox{\boldmath$\theta$\unboldmath}}
\def\Tha{\mbox{\boldmath$\Theta$\unboldmath}}
\def\muu{\mbox{\boldmath$\mu$\unboldmath}}
\def\Si{\mbox{\boldmath$\Sigma$\unboldmath}}
\def\si{\mbox{\boldmath$\sigma$\unboldmath}}
\def\Gam{\mbox{\boldmath$\Gamma$\unboldmath}}
\def\gamm{\mbox{\boldmath$\gamma$\unboldmath}}
\def\Lam{\mbox{\boldmath$\Lambda$\unboldmath}}
\def\De{\mbox{\boldmath$\Delta$\unboldmath}}
\def\vps{\mbox{\boldmath$\varepsilon$\unboldmath}}
\def\Up{\mbox{\boldmath$\Upsilon$\unboldmath}}
\def\xii{\mbox{\boldmath$\xi$\unboldmath}}
\def\Xii{\mbox{\boldmath$\Xi$\unboldmath}}
\def\Lap{\mbox{\boldmath$\LM$\unboldmath}}
\newcommand{\ti}[1]{\tilde{#1}}

\def\tr{\mathrm{tr}}
\def\etr{\mathrm{etr}}
\def\etal{{\em et al.\/}\,}
\newcommand{\indep}{{\;\bot\!\!\!\!\!\!\bot\;}}
\def\argmax{\mathop{\rm argmax}}
\def\argmin{\mathop{\rm argmin}}
\def\vec{\text{vec}}
\def\cov{\text{cov}}
\def\dg{\text{diag}}

\newtheorem{observation}{\textbf{Observation}}
\newtheorem{remark}{Remark}
\newtheorem{theorem}{Theorem}
\newtheorem{lemma}{Lemma}
\newtheorem{definition}{Definition}
\newtheorem{problem}{Problem}
\newtheorem{proposition}{Proposition}
\newtheorem{cor}{Corollary}
\numberwithin{theorem}{section}
\numberwithin{lemma}{section}
\numberwithin{remark}{section}
\numberwithin{cor}{section}

\begin{abstract}
Existing domain adaptation focuses on transferring knowledge between domains with categorical indices (e.g., between datasets A and B).  However, many tasks involve continuously indexed domains. For example, in medical applications, one often needs to transfer disease analysis and prediction across patients of different ages, where age acts as a continuous domain index. Such tasks are challenging for prior domain adaptation methods since they ignore the underlying relation among domains. 
In this paper, we propose the first method for continuously indexed domain adaptation. Our approach combines traditional adversarial adaptation with a novel discriminator that models the encoding-conditioned domain index distribution. Our theoretical analysis demonstrates the value of leveraging the domain index to generate invariant features across a continuous range of domains. Our empirical results show that our approach outperforms the state-of-the-art domain adaption methods on both synthetic and real-world medical datasets\footnote{Code will soon be available at \url{https://github.com/hehaodele/CIDA}}.
\end{abstract}

\section{Introduction}
Machine learning often assumes that training and test data come from the same distribution, so that the trained model generalizes well to the test scenario. This assumption breaks however when the model is trained and tested in distinct domains, i.e., different source and target domains.  Domain adaption (DA) leverages labeled data from the source domains and unlabeled data (or a limited amount of labeled data) from the target domains to significantly improve performance~\cite{bendavid,DANN,ADDA,MDD}.

Existing DA methods however focus on adaption among categorical domains where the domain index is just a label.  A common example would be to adapt a model from one image dataset to another, e.g., adapting from MNIST to SVHN. However, many real-world tasks require adaptation among continuously index domains. For example, in medical applications, one needs to adapt disease diagnosis and prognosis across patients of different ages, where age is a continuous domain index. Treating the age of the source and target domains as domain labels is unlikely to yield the best results because it does not take advantage of the relationship between the disease manifestation and the person's age. Similar issues appear in robotics. For example, underwater robots have to operate at different water depths and viscosity, and one expects that adaptation across datasets from different depths or viscosity (e.g., lake vs. sea) should take into account  the relationship between the robot operation and the physical properties of the liquid in which it operates. These examples highlight the limitations of current DA methods when applied to continuously indexed domains.  

So, how should we perform domain adaption across continuously indexed domains? We note that in the above examples the domain index plays the role of a distance metric -- i.e., it captures a similarity distance between the domains with respect to the task. Thus, one approach for addressing the problem is to modify traditional adversarial adaptation to make the discriminator regress the domain index using a distance-based loss, like the $L_2$ or $L_1$ loss. Although this is better than categorical DA, we show analytically that such treatment can lead to equilibriums with relatively poor domain alignments.  A better solution is to develop a probabilistic discriminator that models the domain index distribution. We show that such a discriminator not only successfully captures the underlying relation among domains, but also enjoys better theoretical guarantees in terms of domain alignment. 
We also note that our method can be naturally generalized to handle multi-dimensional domain indices, achieving further performance gain. For example, in medial applications the index can be a vector of age, blood pressure, activity level, etc.

Our contributions are as follows:
\begin{itemize}
\item We identify the problem of adaptation across continuously indexed domains and propose continuously indexed domain adaptation (CIDA) as the first general DA method for addressing this problem. Further, we analyze our method and provide theoretical guarantees that CIDA aligns continuously indexed domains at equilibrium. 
\item We derive two advanced versions, probabilistic CIDA and multi-dimensional CIDA, to further improve performance and handle multi-dimensional domain indices, with minimal overhead.
\item 
We provide empirical results using both synthetic and real-world medical datasets which show that CIDA and its probabilistic and multi-dimensional variants significantly improve performance over the state-of-the-art DA methods for continuously indexed domains. 
\end{itemize}

\section{Related Work}\label{sec:related}
\textbf{Adversarial Domain Adaptation.} Much prior work has focused on the problem of domain adaptation~\cite{CDANN,CDAN,MCD,GTA,MDD}. The key idea is to match the distributions of the source and target domains. This is achieved by matching their distributions' statistics either directly~\cite{MMD,DDC,CORAL} or with the help of an adversarial loss~\cite{DANN,CDANN,ADDA,MDD,UDA-SGD}. Adversarial domain adaptation is particularly popular due to its relatively strong theoretical insights~\cite{GAN,AMSDA,MDD,InvariantDA} and its compatibility with neural networks. It aligns the distributions of the source and target domains by generating an encoding indistinguishable from a perspective of discriminator that is trained to classify the domain of the data. In this paper, we build on adversarial domain adaptation and extend it to address continuously indexed domains. 

\textbf{Incremental Domain Adaptation.} Closest to our work are incremental DA approaches. Essentially they assume the domain shifts smoothly over time and try to incrementally adapt the source domain to multiple target domains. Different methods are used to perform categorical DA for each domain pair, such as optimal transport~\cite{CDOT}, adversarial loss~\cite{bitarafan2016incremental}, generative adversarial networks~\cite{wulfmeier2018incremental}, and linear transform~\cite{hoffman2014continuous}.  \Citet{CUA} notices such incremental adaptation procedure is prone to catastrophic forgetting, a tendency to forget the knowledge of previous domains while specializing to a new domain, and therefore proposes a replay technique to tackle the issue. Here we note several key differences between CIDA and the methods above. (1) These approaches incrementally perform pair-wise \emph{categorical DA}. Hence failure in adapting one domain pair can lead to catastrophic failures for all following pairs. (2) They only work on DA tasks with one single domain shifting dimension (usually `time'), while our method naturally generalizes to multi-dimensional settings. Such differences are empirically verified in \secref{sec:experiment}.

\section{Methods}
\label{sec:method}
In this section, we formalize the problem of adaptation among continuously indexed domains, and describe our methods for addressing the problem. We then provide theoretical guarantees for the proposed methods in \secref{sec:theory}.

\textbf{Problem.} We consider the unsupervised domain adaptation setting and assume a set of continuous domain indices $\mathcal{U}=\mathcal{U}_s \cup \mathcal{U}_t$, where $\mathcal{U}_s$ and $\mathcal{U}_t$ are the domain index sets for the source and the target domains, and $\mathcal{U}$ is part of a metric space (i.e., a metric like the Euclidian distance is defined over the set). The input and labels are denoted as $\x$ and $y$, respectively. With access to the labeled data $\{(\x_i^s, y_i^s, u_i^s)\}_{i=1}^n$ from source domains ($u_i^s \in \mathcal{U}_s$) and unlabeled data $\{(\x_i^t, u_i^t)\}_{i=1}^m$ from target domains ($u_i^t \in \mathcal{U}_t$), the goal is to accurately predict the labels $\{(y_i^t)\}_{i=1}^m$ for data in the target domains. 

\textbf{Multi-Dimensional Domain Indices.} 
For clarity, we introduce our methods and theory in the context of unidimensional domain indices. However, they can directly apply to multi-dimensional domain indices. Later in \secref{sec:experiment}, we show that the ability of handling multi-dimensional domain indices brings further performance gains.

\subsection{Continuously Indexed Domain Adaptation (CIDA)}\label{sec:cida}
To perform adaptation across a continuous range of domains, we leverage the idea of learning domain-invariant encodings with adversarial training. We propose to learn an encoder\footnote{In general the encoder $E(\x,u)$ can be probabilistic. For example, $\z$ can be generated from a Gaussian distribution whose mean and variance are given by $E(\x,u)$.} $E$ and a predictor $F$ such that the distribution of the encodings $\z = E(\x) \in \ZM$ (or $\z = E(\x,u)$) from all domains $\mathcal{U}$ are aligned so that all labels can be accurately predicted by the shared predictor $F$. 
Formally, domain-invariant encodings require that $p(\z|u_1)=p(\z|u_2),\forall u_1,u_2\in\UM$. It implies that $\z$ and $u$ are independent ($u \indep \z$), i.e., $p(u | \z) = p(u)$ or equivalently $p(\z | u) = p(\z)$. This is achieved with the help of a discriminator $D$. In continuously indexed domains however, small changes in $u$ should lead to small changes in the encoding. Thus, instead of classifying the encoding into categorical domains, the discriminator $D$ in CIDA regresses the domain index. 

Formally, CIDA performs a minimax optimization with the value function $V(E,F,D)$ as:
\begin{align}
&\min\limits_{E,F} \max\limits_D V_p(E,F) - \lambda_d V_d(D,E), \label{eq:full_game}
\end{align}
where we have
\begin{align}
& V_p(E,F) \triangleq \EB^s [L_p(F(E(\x,u)),y)] \nonumber \\
& V_d(D,E) \triangleq \EB[L_d( D(E(\x, u)),u)] \nonumber
\end{align}
where $\EB$ and $\EB^s$ denote the expectations taken over the entire data distribution $p(\x,y,u)$ and the source data distribution $p^s(\x,y,u)$. Note that the label $y$ is only accessible in the source domains. $L_p$ is the prediction loss (e.g., cross-entropy loss for classification tasks), and $L_d$ is the domain index loss. $\lambda_d$ is a hyperparameter balancing both losses. The main difference between CIDA and traditional adversarial domain adaptation is that the discriminator loss $L_d$ is a monotonic function of the metric defined over $\mathcal{U}$.

\subsection{Variants of CIDA}\label{sec:pcida}
Note that there can be various designs for both $D$ and $L_d$. For example, $D$ can either directly predict the domain index or predict its mean and variance, and $L_d$ can be either the $L_2$ or $L_1$ loss. Different designs come with different theoretical guarantees. 

\textbf{Vanilla CIDA.} In the vanilla CIDA, $D$ directly predicts the domain index, and correspondingly $L_d$ is the $L_2$ loss between the predicted and ground-truth domain index,
\begin{align}
\label{eq:l2_loss}
L_d(D(\z),u) = (D(\z) - u)^2,
\end{align}
Vanilla CIDA above only guarantees matching the mean of the distribution $p(u|\z)$ (see theoretical results in \secref{sec:theory}). 

Therefore in the following, we introduce an advanced version, dubbed probabilistic CIDA (PCIDA), which enjoys better theoretical guarantees to match both the mean and variance of the distribution $p(u|\z)$. We note that PCIDA can be extended to match higher-order moments.

\textbf{Probabilistic CIDA.} The major improvement from CIDA to PCIDA is that in PCIDA, the discriminator predicts the distribution of $p(u|\z)$ instead of providing point estimation. We start with the simplest probabilistic model, Gaussian distributions, where the discriminator $D$ outputs the mean and variance of $p(u|\z)$ as $D_{\mu}(\z)$ and $D_{\sigma^2}(\z)$, respectively. To train such a discriminator, we use the negative log-likelihood as the loss function:
\begin{align}
\label{eq:guass_loss}
L_d(D(\z),u) = \frac{(D_{\mu}(\z) - u)^2}{2 D_{\sigma^2}(\z)} + \frac{1}{2}\log D_{\sigma^2}(\z),
\end{align}
where $D(\z)=(D_{\mu}(\z), D_{\sigma^2}(\z))$. 

\textbf{Extension to Gaussian Mixture Models.} PCIDA can be naturally extended from a single Gaussian to a Gaussian mixture model (GMM) by using a mixture density network as the discriminator $D$~\cite{MDN} and the corresponding negative log-likelihood as $L_d(\cdot,\cdot)$. 
\section{Theoretical Results}\label{sec:theory}
In this section, we provide theoretical guarantees for CIDA and PCIDA. As standard in adversarial domain adaption, we analyze a game in which the encoder aims to fool the discriminator and prevent it from inferring the domain index.  We first analyze a simplified game between the encoder and the discriminator (without the predictor) to gain insight of the aligned encodings. We then discuss the full three-player game and show our framework preserves the prediction power while aligning the encodings.

\subsection{Analysis for the Simplified Game}
We consider a simplified game which does not involve the predictor $F$, defined by the $V_d(E,D)$ term in \eqnref{eq:full_game}:
\begin{align}
\max\limits_{E} \min \limits_D V_d(E,D) = \mathbb{E} [L_d( D(E(\x, u)),u)]. \label{eq:simplified_game}
\end{align}

We first analyze the equilibrium of the simplified game for CIDA.
Recall that, in CIDA, the discriminator $D$ predicts the domain index $u$ given the encoding $\z$ and the domain index loss $L_d$ is the $L_2$ loss. We show that in the equilibrium of CIDA, the encoder will align the mean of the conditional domain distribution $p(u|\z)$ to the mean of the marginal domain distribution $p(u)$.

\lemref{lem:opt_dis_cida} below analyzes the discriminator $D$ with the encoder $E$ fixed and states that the optimal discriminator $D$ outputs the mean domain index of all data with the same encoding $\z$. 
\begin{lemma}[\textbf{Optimal Discriminator for CIDA}]\label{lem:opt_dis_cida}
For E fixed, the optimal D is
\begin{align*}
D^*_{E}(E(\x,u)) = \mathbb{E}_{u\sim p(u | \z)} [u],
\end{align*}
where $\z=E(\x,u)$.
\end{lemma}
\begin{proof}
With E fixed, the optimal D
\begin{align*}
D^*_{E} &= \argmin\limits_{D} \mathbb{E}_{(\x,u)\sim p(\x,u)} [L_d(D(E(\x, u)),u)]\\
&= \argmin\limits_{D} \mathbb{E}_{(\z,u)\sim p(\z,u)} [\|D(\z)-u\|_2^2]\\
&=\argmin\limits_{D} \mathbb{E}_{\z\sim p(\z)} \mathbb{E}_{u\sim p(u | \z)}[\|D(\z)-u\|_2^2]
\end{align*}
Notice that
\begin{align*}
&\mathbb{E}_{u\sim p(u | \z)}[(D(\z) - u)^2] \\ =&\mathbb{E}_{u\sim p(u | \z)}[u^2] - 2D(\z)\mathbb{E}_{u\sim p(u | \z)}[u] + D(\z)^2,
\end{align*}
is a quadratic form of $D(\z)$ which achieves the minimum at $D(\z)=\mathbb{E}_{u\sim p(u | \z)}[u]$.
\end{proof}

Assuming that $D$ always achieves its optimum w.r.t $E$ during the training, the minimax game in \eqnref{eq:simplified_game} can be reformulated as maximizing $C_d(E)$ where
\begin{align*}
C_d(E) & \triangleq \min\limits_D V_d(E,D) = V_d(E,D^*_E)\\
&= \mathbb{E}_{(\x,u)\sim p(\x,u)} (\mathbb{E}_{u\sim p(u | \z)} [u] - u)^2 \\
&= \mathbb{E}_{\z\sim p(\z)} \mathbb{E}_{u \sim p(u | \z)} (\mathbb{E}_{u\sim p(u | \z)} [u] - u)^2\\
&= \mathbb{E}_{\z\sim p(\z)} \VB_{u \sim p(u | \z)}[u] = \EB_\z \VB[u|\z],
\end{align*}
where $\VB$ denotes variance.

Next we analyze the virtual training criterion $C_d(E)$ for the encoder and derive the global optimum. 

\begin{lemma}[\textbf{Uniqueness of Constant Expectation}]\label{lem:unique_l2}
If there exists a constant $\mu_c$ such that $\mathbb{E}_{u\sim p(u | \z)}[u] = \mu_c$ for any $\z$, we have $\mu_c = \mathbb{E}_{u\sim p(u)}[u]$.
\end{lemma}

\begin{theorem}[\textbf{Global Optimum for CIDA}]\label{thm:iff_l2}
The global maximum of $C_d(E)$ is achieved if and only if the encoder $E$ satisfies that the expectations of the domain index $u$ over the conditional distribution $p(u|\z)$ for any given $\z$ are identical to the expectation over the marginal distribution $p(u)$, i.e., $\EB[u|\z]=\EB[u],\forall \z$.
\end{theorem}
\begin{proof}
We first show $C_d(E) \leq \VB[u]$ and then show the equality is achieved when $\EB[u|\z]=\EB[u],\forall \z$.
\begingroup\makeatletter\def\f@size{9}\check@mathfonts
\begin{align*}
C_d(E) - \VB[u] 
&= \EB_\z \VB[u|\z] - \VB[u] \\
&= \EB_\z [\EB[u^2|\z] - \EB[u|\z]^2] - (\EB[u^2] - \EB[u]^2)\\
&= \EB[u]^2 - \EB_\z[\EB[u|\z]^2].
\end{align*}
\endgroup
By the convexity of $x^2$ and Jensen's inequality, we have $\EB[u]^2=(\EB_\z[\EB[u|\z]])^2 \leq \EB_\z[\EB[u|\z]^2] $ and the equality is achieved when $\EB[u|\z]$ is constant w.r.t. $\z$. By \lemref{lem:unique_l2} we have $\EB[u|\z]=\EB[u],\forall \z$.
\end{proof}

As \thmref{thm:iff_l2} states, the vanilla CIDA using the $L_2$ loss guarantees that the mean of the distribution $p(u| \z)$ matches the mean of the marginal distribution $p(u)$.
It means that there is a risk the encoder $E$ only aligns the mean of the distributions without exactly matching the entire distributions. However, surprisingly, we find that CIDA often achieves good empirical performance (see \secref{sec:experiment} for more details). Next, we analyze PCIDA and show that PCIDA enjoys better theoretical guarantees and matches both the mean and variance of the distribution $p(u|\z)$.

Recall that in PCIDA, the discriminator $D$ outputs the mean and variance of $p(u|\z)$ as $D_{\mu}(\z)$ and $D_{\sigma^2}(\z)$. We use the negative log-likelihood (\eqnref{eq:guass_loss}) as the domain loss $L_d$. We start from analyzing the discriminator $D$ when the encoder $E$ is fixed. \lemref{lem:opt_dis_pcida} states that the optimal discriminator $D$, given the encoding $\z$, will output the mean and variance of the domain index distribution $p(u | \z)$ .

\begin{lemma}[\textbf{Optimal Discriminator for PCIDA}]\label{lem:opt_dis_pcida} With E fixed, the optimal D is
\begin{align*}
D^*_{{\mu},E}(\z) &= \mathbb{E}_{u\sim p(u | \z)} [u],\\
D^*_{{\sigma^2},E}(\z) &= \mathbb{V}_{u\sim p(u | \z)} [u],
\end{align*}
where $\z=E(\x,u)$, and $D=(D_{\mu}, D_{\sigma^2})$.
\end{lemma}
Proof of \lemref{lem:opt_dis_pcida} is similar to that of \lemref{lem:opt_dis_cida} (see the Supplement for details).

Assuming discriminator $D$ always reaches optimum, the virtual training criterion $C_d(E)$ for the encoder becomes:
\begin{align*}
C_d(E) &= \min\limits_D V_d(E,D) \nonumber = V_d(E,D^*_E)\\
&= \EB_{\z,u} \left[\frac{(\EB[u|\z] - u)^2}{2\VB[u|\z]} + \frac{1}{2} \log(\VB[u|\z]) \right].
\end{align*}
Now we analyze $C_d(E)$ and provide PCIDA's global optimum.

\begin{lemma}[\textbf{Uniqueness of Constant Expectation and Variance}]\label{lem:unique}
If there exist constants $\mu_c$ and $\sigma_c^2$ such that $\mathbb{E}_{u\sim p(u | \z)}[u] = \mu_c$ and $\mathbb{V}_{u\sim p(u | \z)}[u] = \sigma_c^2$ for any $\z$, we have $\mu_c = \mathbb{E}_{u\sim p(u)}$ and $\sigma_c^2 = \mathbb{V}_{u\sim p(u)}[u]$.
\end{lemma}

\begin{theorem}[\textbf{Global Optimum for PCIDA}]\label{thm:iff_gaussian}
In PCIDA (with the Gaussian model), the global optimum is achieved if and only if the mean and variance of the distribution $p(u| \z)$ given any $\z$ are identical to those of the marginal distribution $p(u)$.
\end{theorem}

\begin{proof}
Given that
\begin{align*}
C_d(E) = \underbrace{\EB_{\z,u}\left[\frac{(\EB[u|\z] - u)^2}{2\VB[u|\z]}\right]}_{C_1} + \underbrace{\EB_{\z,u}\left[\frac{1}{2} \log(\VB[u|\z])\right]}_{C_2},
\end{align*}
we analyze the upper bounds of the two terms separately.
For the first term, 
\begingroup\makeatletter\def\f@size{9}\check@mathfonts
\begin{align*}
C_1
&=\EB_{\z}\EB_{u|\z}\left[\frac{(\EB[u|\z] - u)^2}{2\VB[u|\z]}\right]=\EB_{\z}\left[\frac{\EB_{u|\z}(\EB[u|\z] - u)^2}{2\VB[u|\z]}\right]\\
&=\EB_{\z}\left[\frac{\VB[u|\z]}{2\VB[u|\z]}\right]=\EB_{\z} \frac{1}{2}=\frac{1}{2}.
\end{align*}
\endgroup
For the second term, by the concavity of $\log(x)$ and Jensen's inequality, we have that $2 C_2 \leq \log(\EB_\z [\VB[u|\z]])$ the equality holds when $\VB[u|\z]$ is constant w.r.t. $\z$. Further, in the proof of \thmref{thm:iff_l2}, we show that $\EB_\z [\VB[u|\z]] \leq \VB[u]$ and the maximum is achieved when $\EB[u|\z]$ is constant w.r.t. $\z$. 
Together with \lemref{lem:unique}, we then have that $C_d(E)$ reaches the global optimal $0.5 + 0.5 \log(\VB[u])$ if and only if $\EB[u|\z]=\EB[u]$ and $\VB[u|\z]=\VB[u]$ for all $\z$.
\end{proof}

\begin{cor}\label{thm:achieve_cida}
For both CIDA and PCIDA, the global optimum of $C_d(E)$ is achieved if the encoding of all domains (continuously indexed by $u$) are totally aligned, i.e., $\z \indep u$.
\end{cor}

\begin{remark}[\textbf{Matching Higher-Order Moments}]
By \thmref{thm:iff_l2} and \thmref{thm:iff_gaussian}, we show that CIDA using the $L_2$ loss matches the mean of $p(u | \z)$ while the PCIDA with the Gaussian model matches both the mean and variance of $p(u | \z)$. Can we match higher-order moments? We believe our methodology can generalize to match higher-order moments by using PCIDA with more complex parametric probabilistic models. For example, one can use skew-normal distributions~\cite{azzalini2013skew} to match the third moment (skewness). 

\end{remark}

\subsection{Analysis of the Three-player Game}

We analyze the equilibrium state of the three-player game of $E, F$ and $D$ as defined in \eqnref{eq:full_game}. 
We divide the situation into two cases based on whether the domain index $u$ is independent of the label $y$.

\subsubsection{$u \indep y$}
The domain index $u$ is independent of the label $y$ when it captures nuisance variations that are irrelevant to the task of predicting the label $y$. In this case, we prove the following theorem showing that the optimal encoding captures all the information in the input $x$ that is relevant to the predictive tasks while aligning the domain index distributions.

\begin{lemma}[\textbf{Optimal Predictor}]
\label{thm:opt_predictor}
Given the encoder $E$, the prediction loss $V_p(F,E)\triangleq L_p(F(E(\x,u)),y) \geq H(y|E(\x,u))$ where $H(\cdot)$ is the entropy. The optimal predictor $F^*$ that minimizes the prediction loss is $F^*(E(\x,u))=P_y(\cdot|E(\x,u))$.
\end{lemma}

Assuming the predictor $F$ and the discriminator $D$ are trained to achieve their optimal losses, by \lemref{thm:opt_predictor}, the three-player game (\eqnref{eq:full_game}) can be rewritten as following training procedure of the encoder $E$,
\vspace{-1.5mm}
\begin{align}
\min_E C(E) \triangleq H(y|E(\x,u)) - \lambda_d C_d(E).
\end{align}

\begin{theorem} \label{thm:full_game} If the encoder $E$, the predictor $F$ and the discriminator $D$ have enough capacity and are trained to reach optimum, any global optimal encoder $E^*$ has the
following properties:
\begin{subequations}
	\begin{align}
	H(y|E^*(\x, u)) = H(y|\x, u) \label{optimal-encoder-a}\\
	C_d(E^*) = \max_{E'} C_d(E') \label{optimal-encoder-b}
	\end{align}
\end{subequations}
\end{theorem}
\begin{proof}
	Since $E(\x, u)$ is a function of $\x, u$, by the data processing inequality, we have $H(y|E(\x, u)) \geq H(y|\x, u)$.

	Hence, $C(E) = H(y|E(\x, u)) - \lambda_d C_d(E) \geq H(y|\x,u) - \lambda_d \max_{E'} C_d(E')$. The equality holds if and only if $H(y|\x,u)=H(y|E(\x,u))$ and $C_d(E)=\max_{E'} C_d(E')$. Therefore, we only need to prove that the optimal value of $C(E)$ is equal to $H(y|\x,u) - \lambda_d \max_{E'} C_d(E')$ in order to prove that any global encoder $E^*$ satisfies both \eqnref{optimal-encoder-a} and \eqnref{optimal-encoder-b}.
	
	We show that $C(E)$ can achieve $H(y|\x,u) - \lambda_d \max_{E'} C_d(E')$ by considering the following encoder $E_0$: $E_0(\x,u) = P_y(\cdot|\x, u)$. It can be examined that $H(y|E_0(\x,u)) = H(y|\x,u)$ and $E_0(\x,u) \indep u$ which leads to $C(E_0)=\max_{E'}C(E')$ using \crlref{thm:achieve_cida}.
\end{proof}

\thmref{thm:full_game} shows that, at the equilibrium, the optimal encoder preserves all the information about label $y$ contained in the data $\x$ and the domain index $u$ while aligning the encoding cross domains.

Note that in general the encoder $E$ is a probabilistic encoder that generates $\z$ stochastically. For example, one can use a probabilistic encoder parameterized by a natural-parameter network~\cite{NPN} and generate $\z$ using the reparameterization trick~\cite{VAE}. Empirically, we find that directly using a deterministic encoder also works favorably and therefore keep the encoder deterministic in~\secref{sec:experiment} for simplicity. 

\subsubsection{$u \not\!\perp\!\!\!\perp y$}
The domain index $u$ is dependent of the label $y$ when it contains information relevant to predicting $y$.
In this case, discretization of the inherently continuous domain index $u$ is necessary to perform categorical domain adaptation. However, this discretization inevitably loses information in $u$ and could hurt the predictive task since $u \not\!\perp\!\!\!\perp y$. In contrast, our methods CIDA/PCIDA performs domain adaption with the continuous domain index $u$, thus, can fully retain information in $u$ that is relevant to the label $y$.
\section{Experiments}\label{sec:experiment}
We evaluate CIDA and its variants on two toy datasets, one image dataset (\emph{Rotating MNIST}), and three real-world medical datasets. These empirical studies verify our theoretical findings in \secref{sec:method} and show that:
\begin{Itemize}
\item Using categorical domain adaption to align continuously indexed domains leads to poor alignment with marginal (or no) performance gain compared to no adaptation.
\item CIDA aligns domains with continuous indices and achieves significant performance boost compared to categorical domain adaption methods.
\item PCIDA's ability to predict a distribution rather than a single value is helpful in avoiding bad equilibriums and improving prediction performance.
\item  The performance gains of CIDA and PCIDA increase with multi-dimensional domain indices.
\end{Itemize}

\begin{figure*}[!tb]
\vspace{-0pt}
\centering     
\subfigure[Domains]{\includegraphics[width=0.23\textwidth]{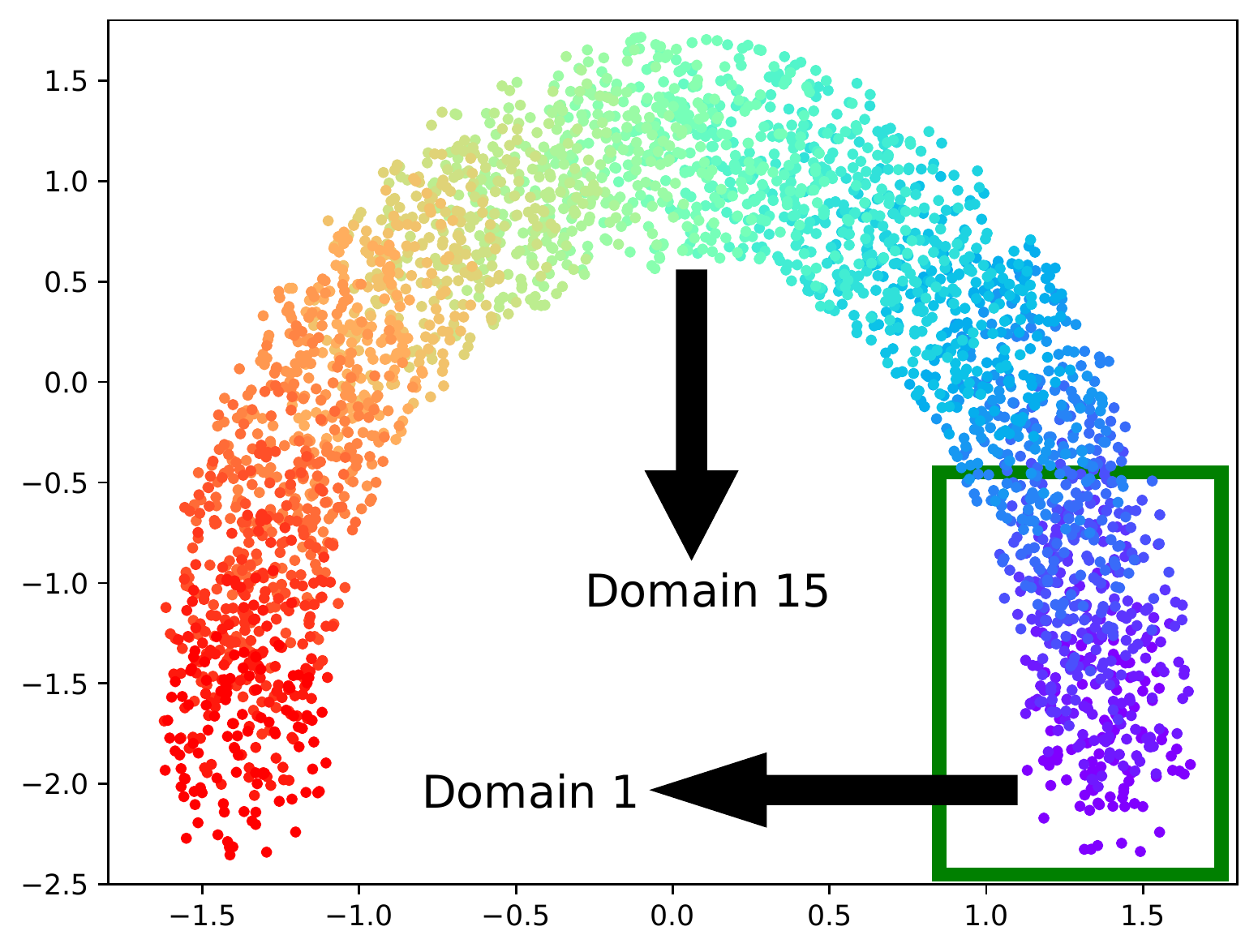}}
\subfigure[Ground Truth]{\includegraphics[width=0.23\textwidth]{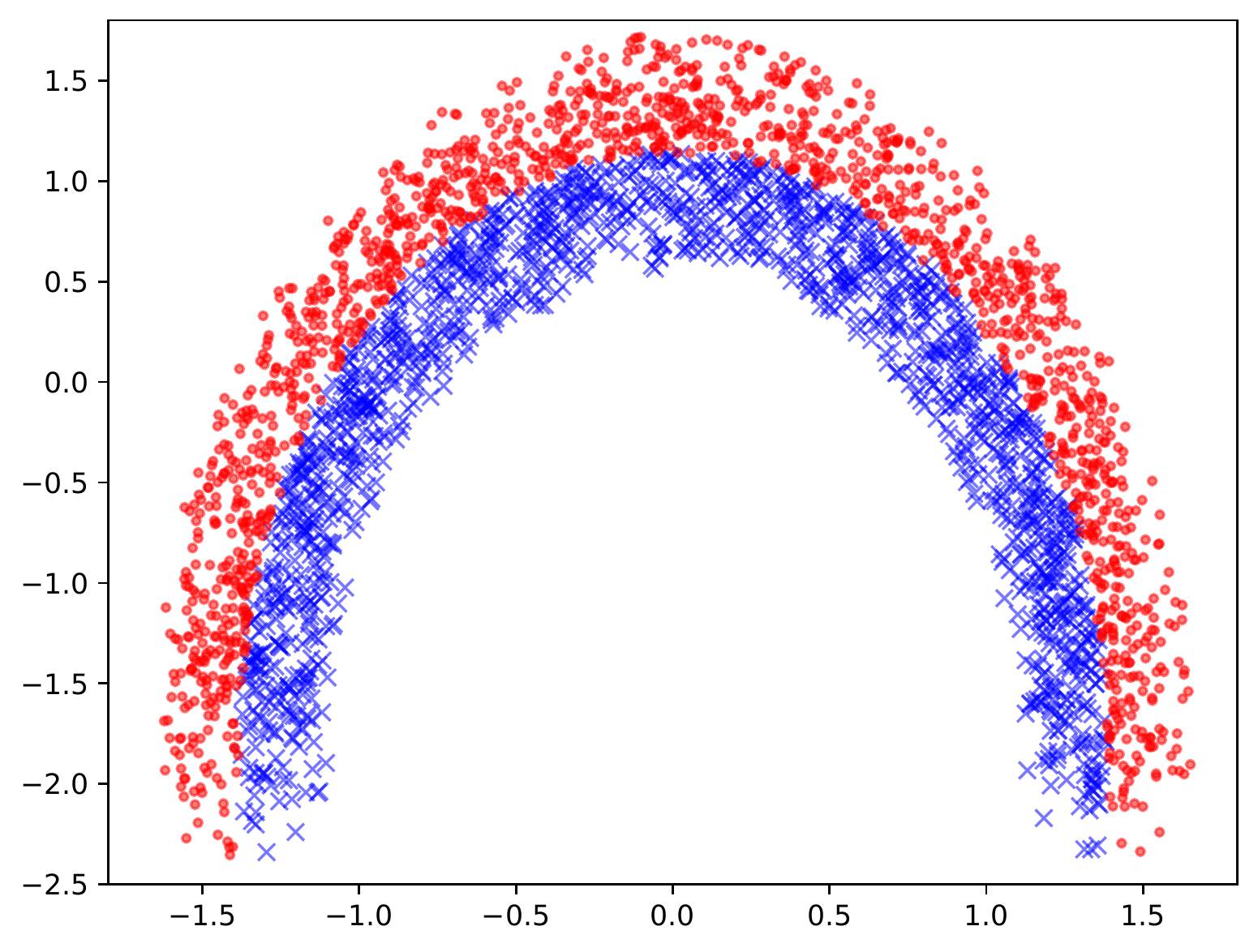}}
\subfigure[DANN]{\includegraphics[width=0.23\textwidth]{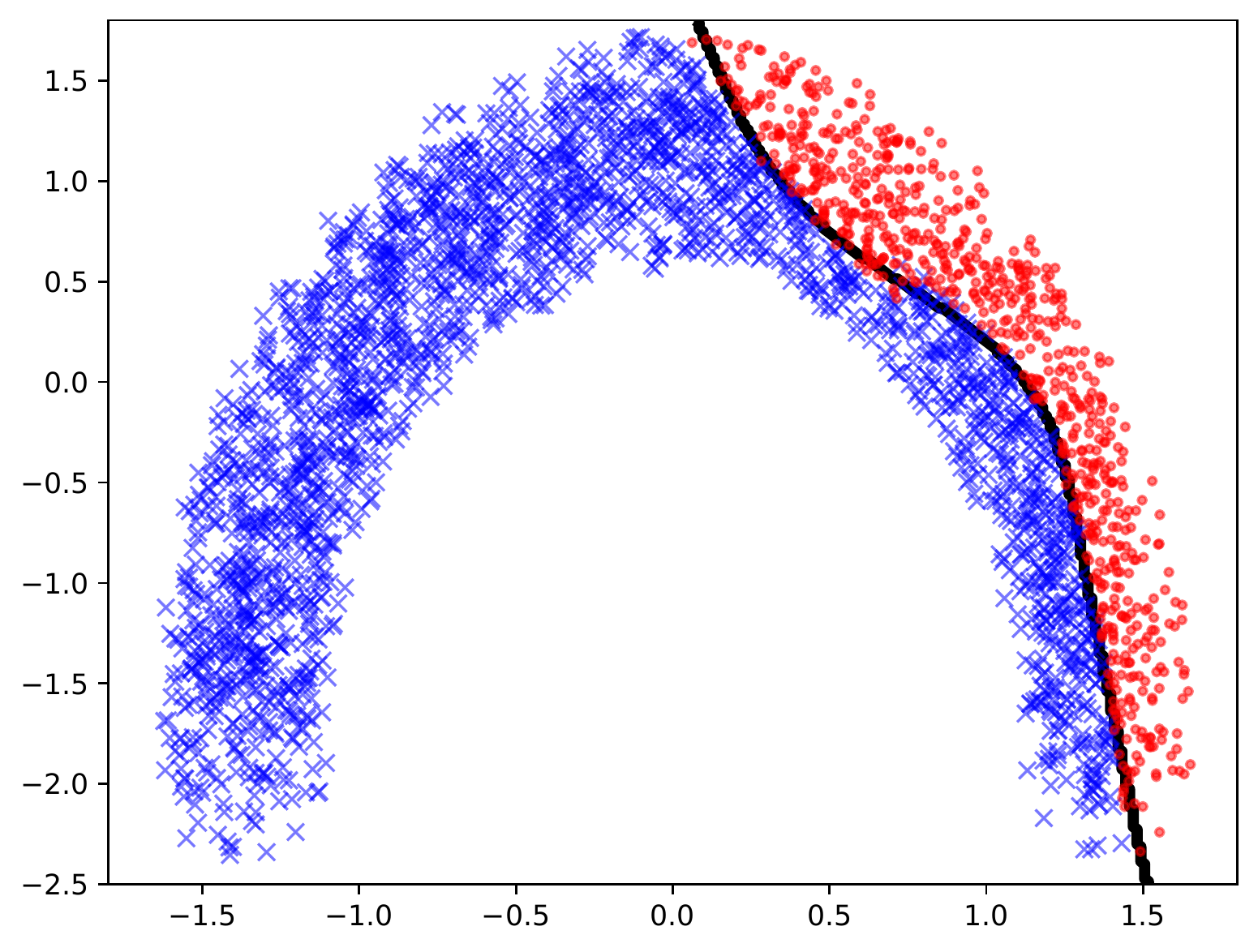}}
\subfigure[CDANN]{\includegraphics[width=0.23\textwidth]{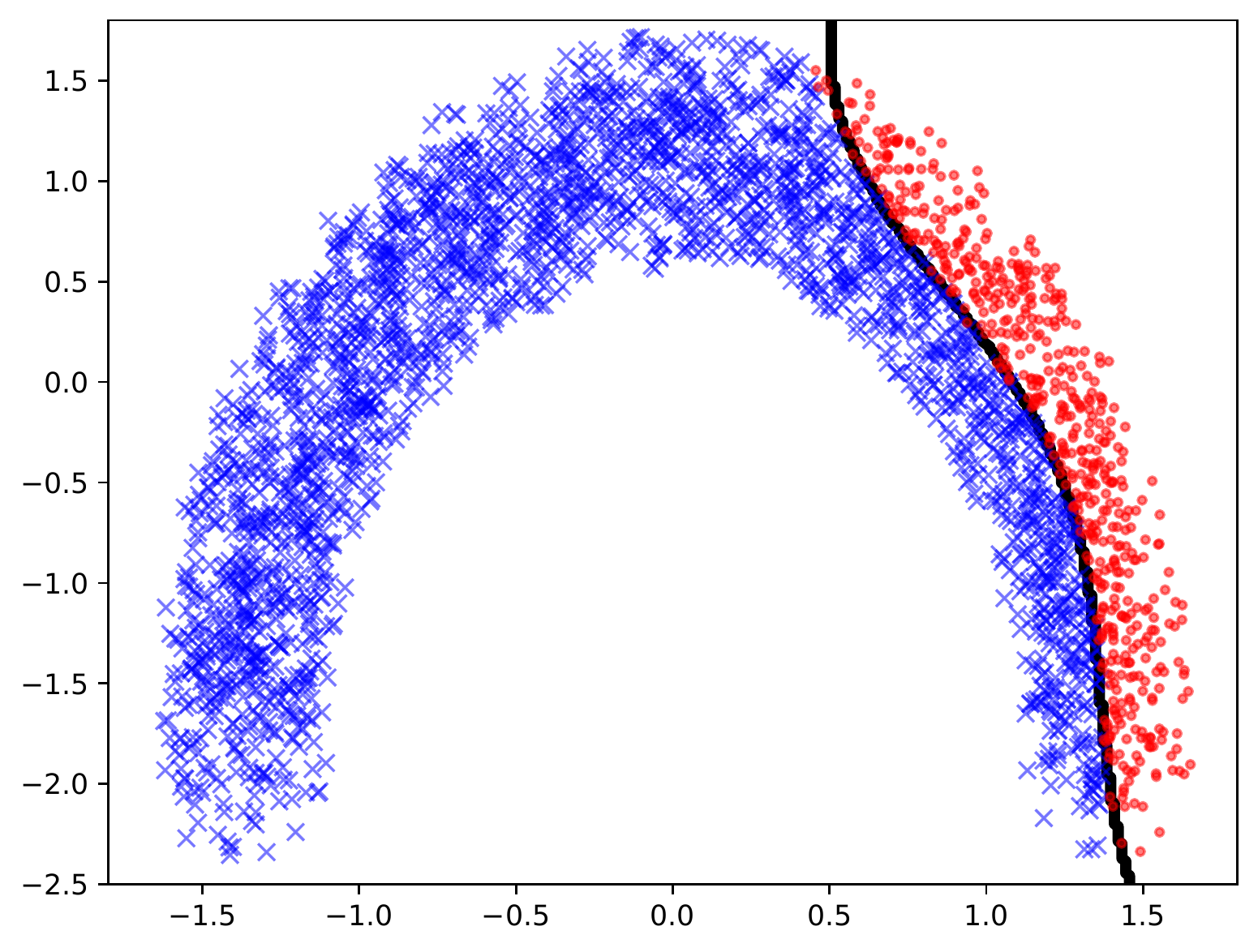}}
\vskip -0.3cm
\subfigure[ADDA]{\includegraphics[width=0.23\textwidth]{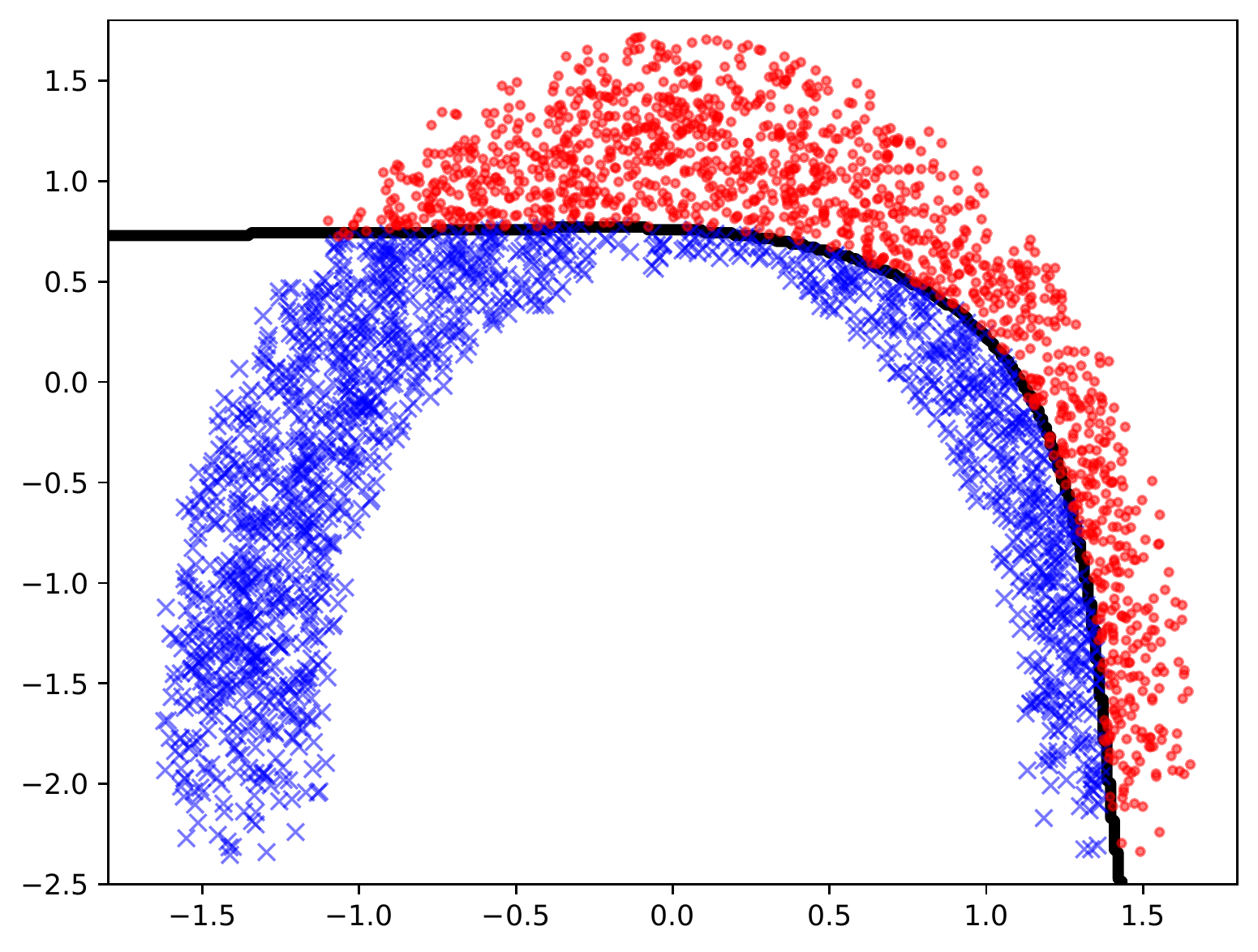}}
\subfigure[MDD]{\includegraphics[width=0.23\textwidth]{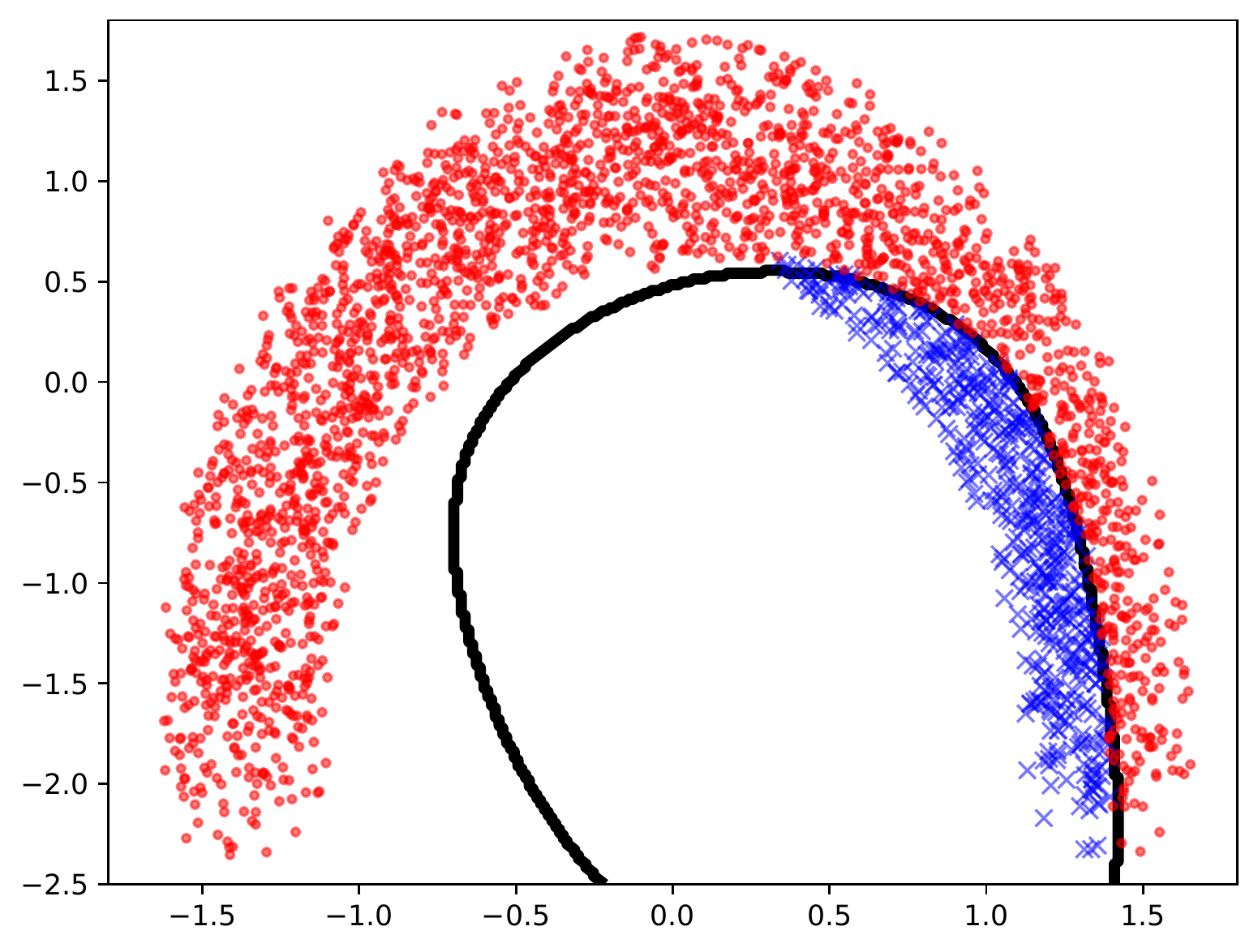}}
\subfigure[CUA]{\includegraphics[width=0.23\textwidth]{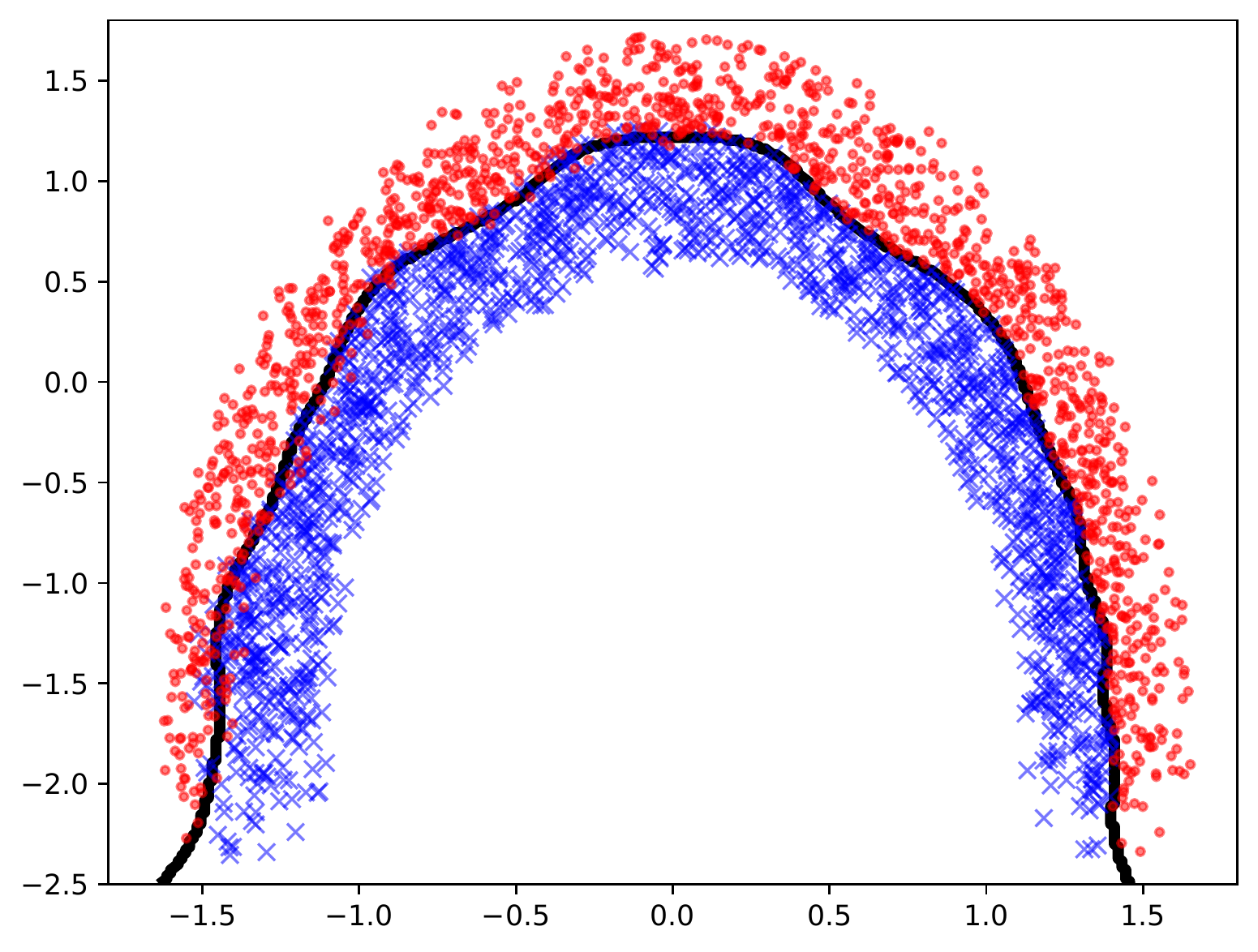}}
\subfigure[CIDA (Ours)]{\includegraphics[width=0.23\textwidth]{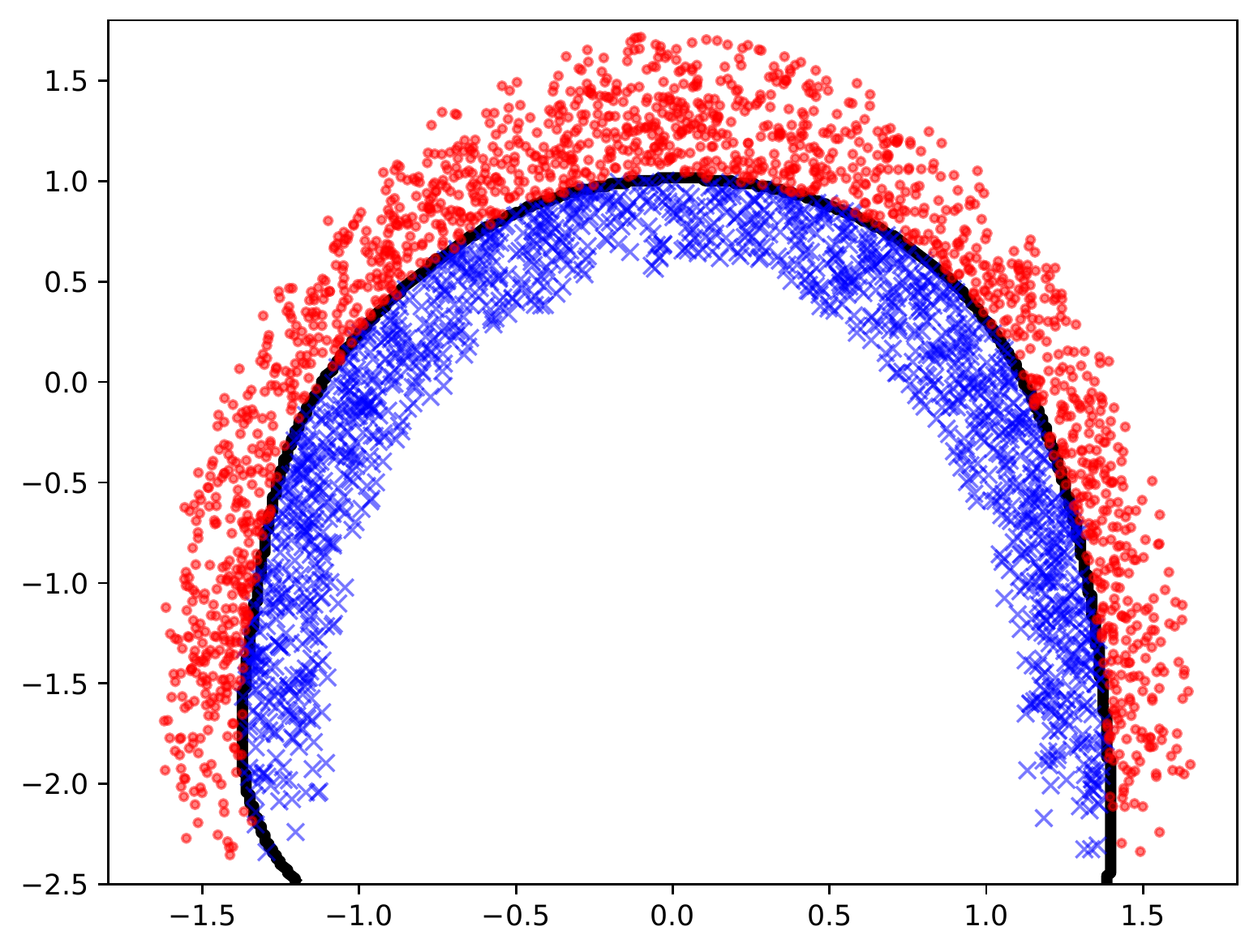}}
\vspace{-12pt}
\caption{Results on the \emph{Circle} dataset with $30$ domains. \figref{fig:toy-circle}(a) shows domain index by color. The first $6$ domains are source domains, marked by green boxes. Red dots and blue crosses are positive and negative data samples. Black lines show the decision boundaries generated according to model predictions.}
\label{fig:toy-circle}
\vspace{-10pt}
\end{figure*}

\begin{figure*}[!tb]
\vspace{-0pt}
\centering     
\subfigure[Domains]{
\includegraphics[width=0.23\textwidth]{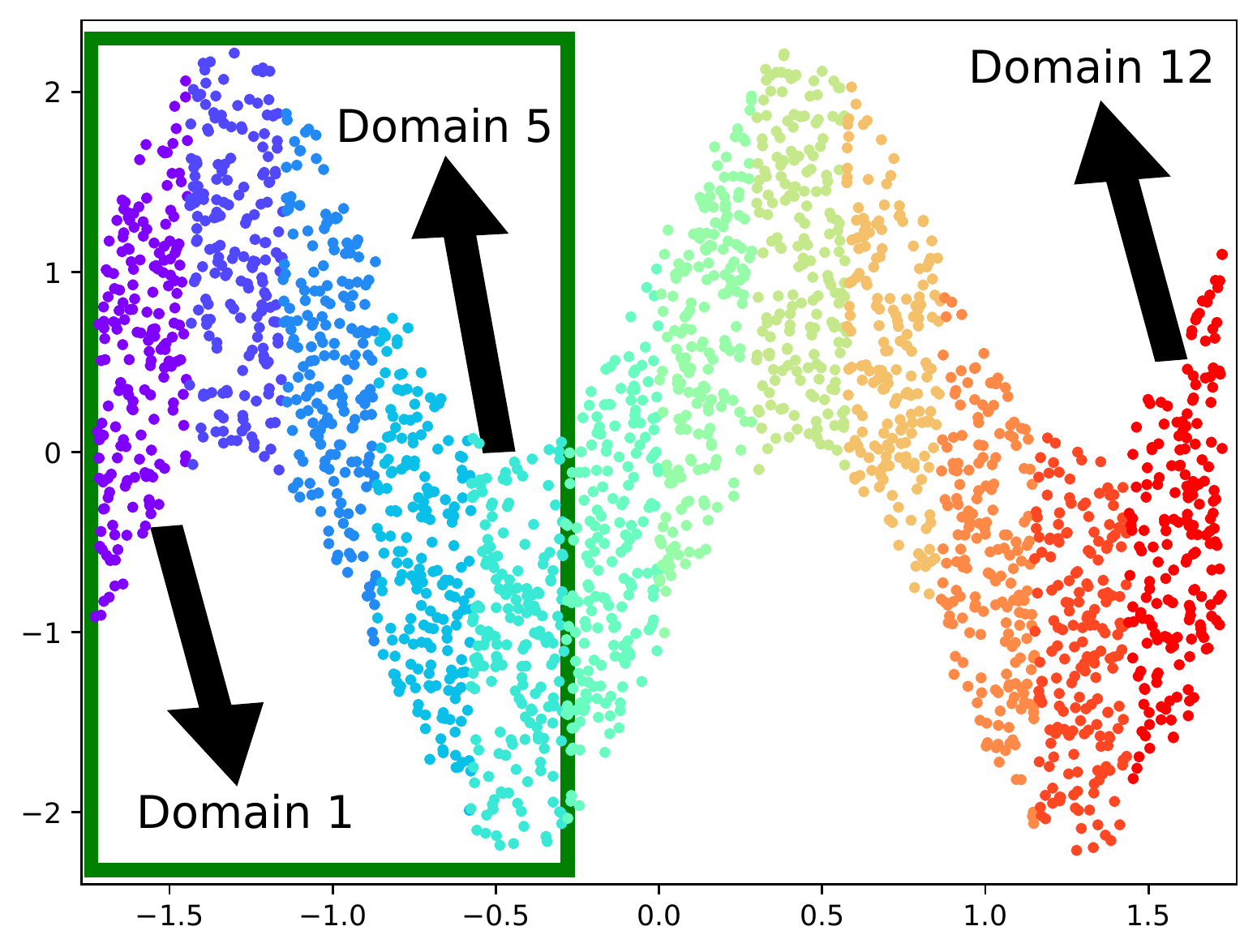}}
\subfigure[Ground Truth]{\includegraphics[width=0.23\textwidth]{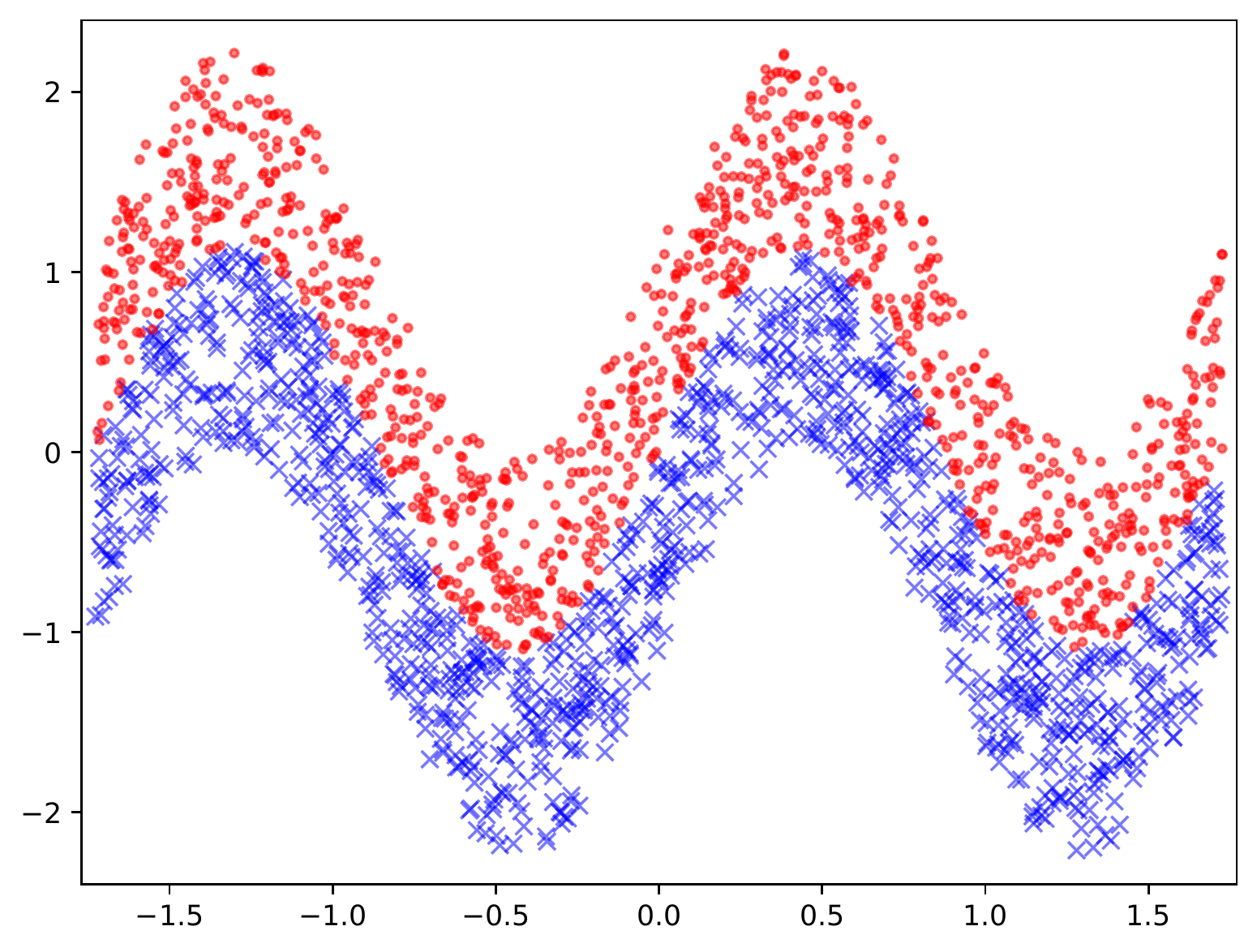}}
\subfigure[DANN]{\includegraphics[width=0.23\textwidth]{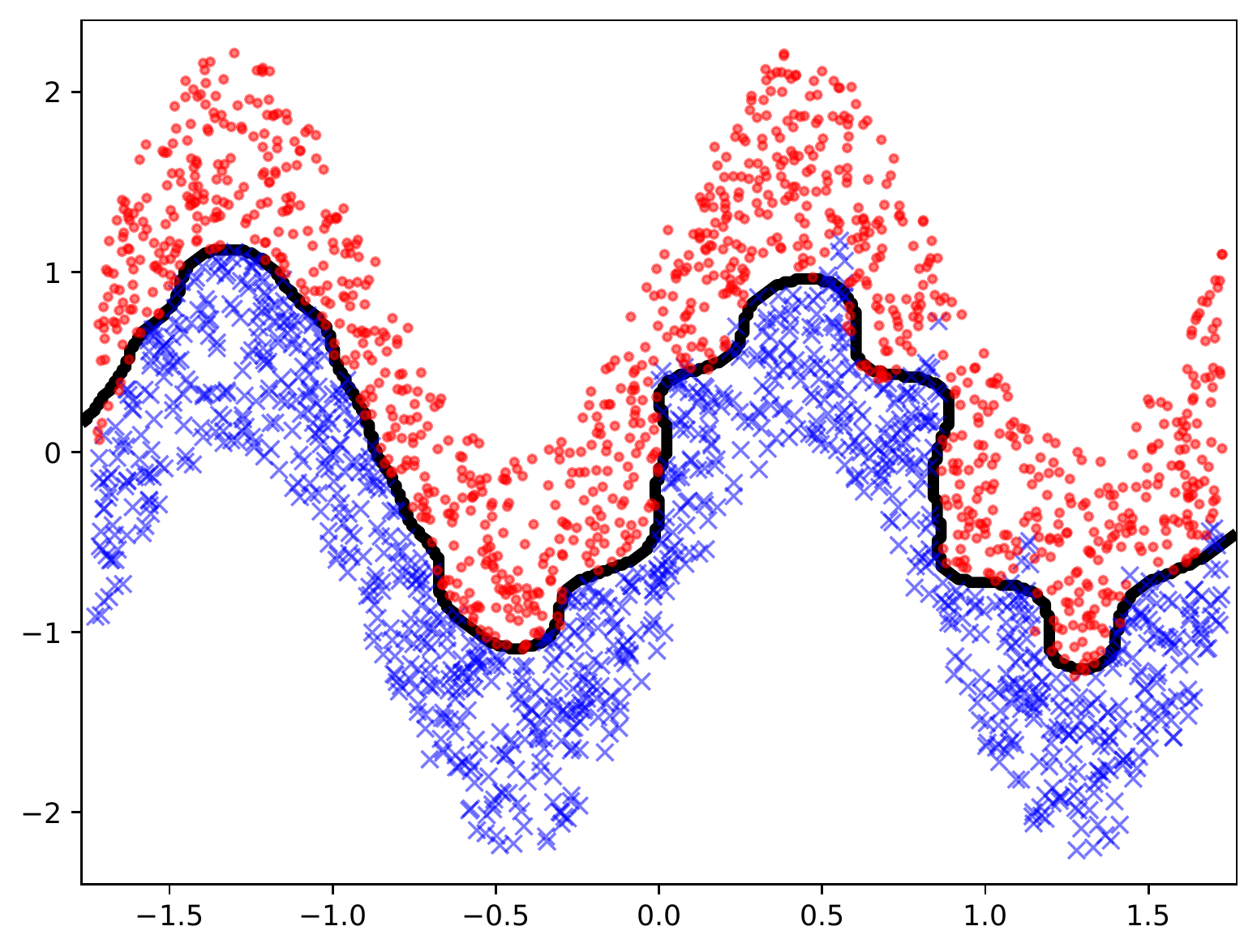}}
\subfigure[CDANN]{\includegraphics[width=0.23\textwidth]{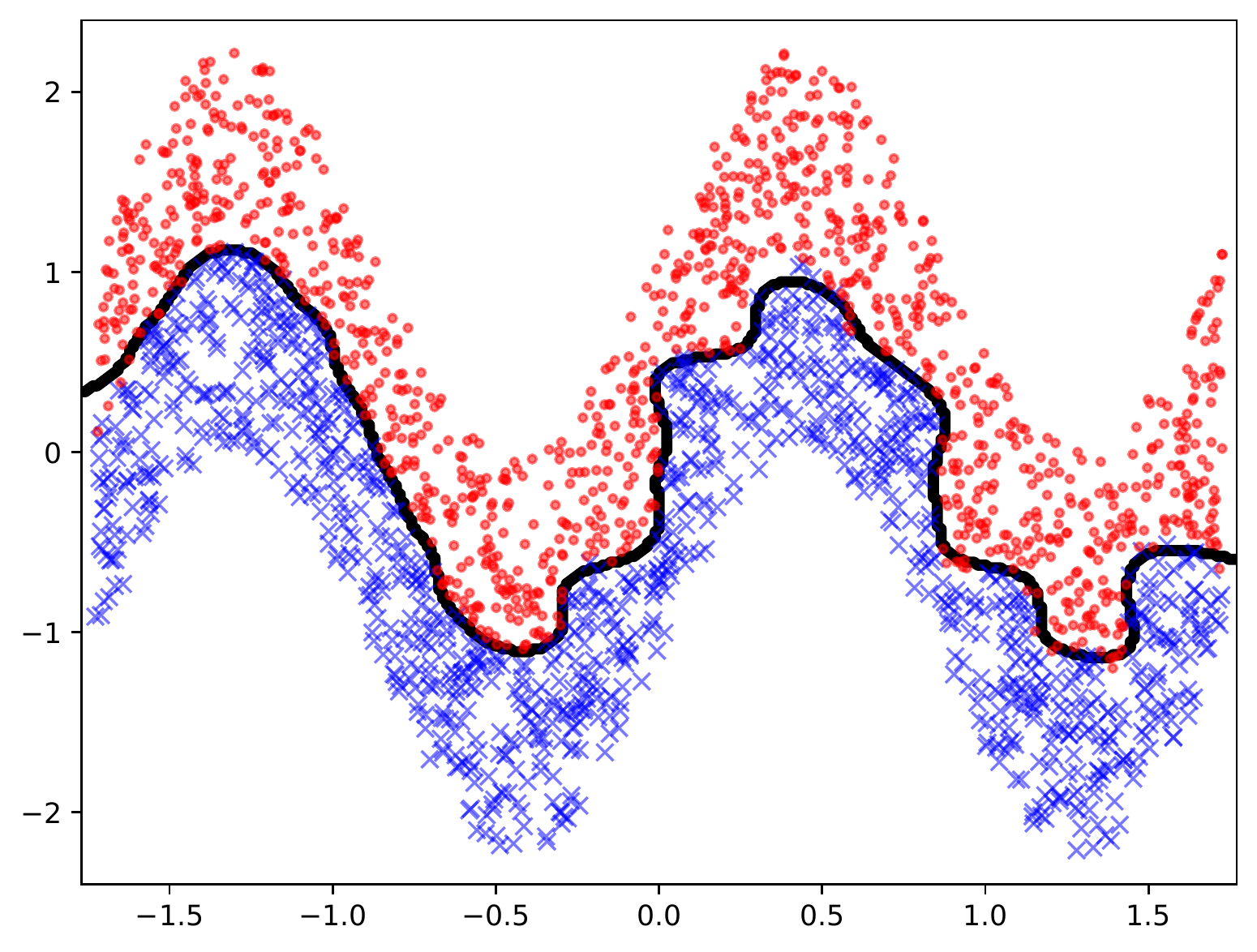}}
\vskip -0.3cm
\subfigure[ADDA]{\includegraphics[width=0.23\textwidth]{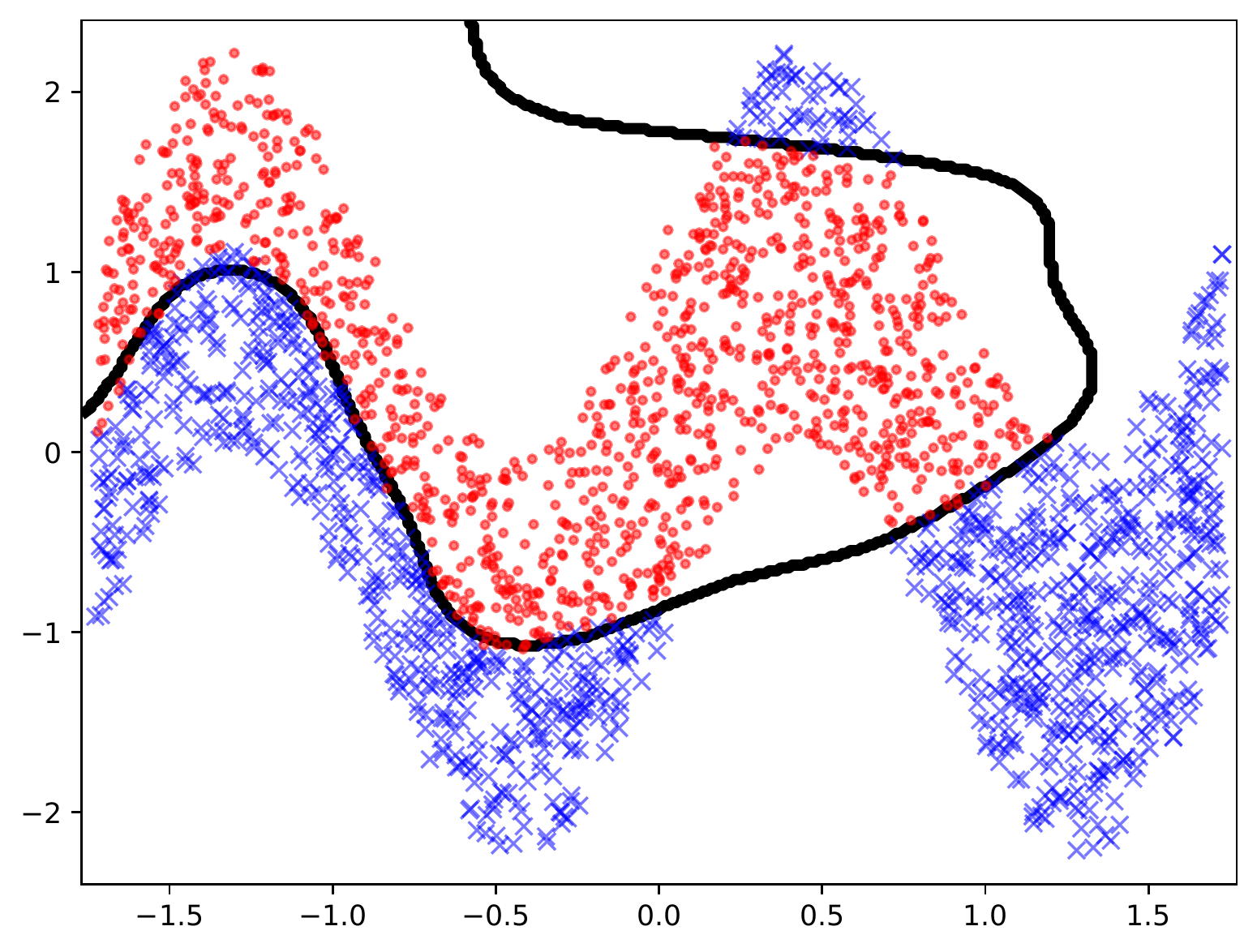}}
\subfigure[MDD]{\includegraphics[width=0.23\textwidth]{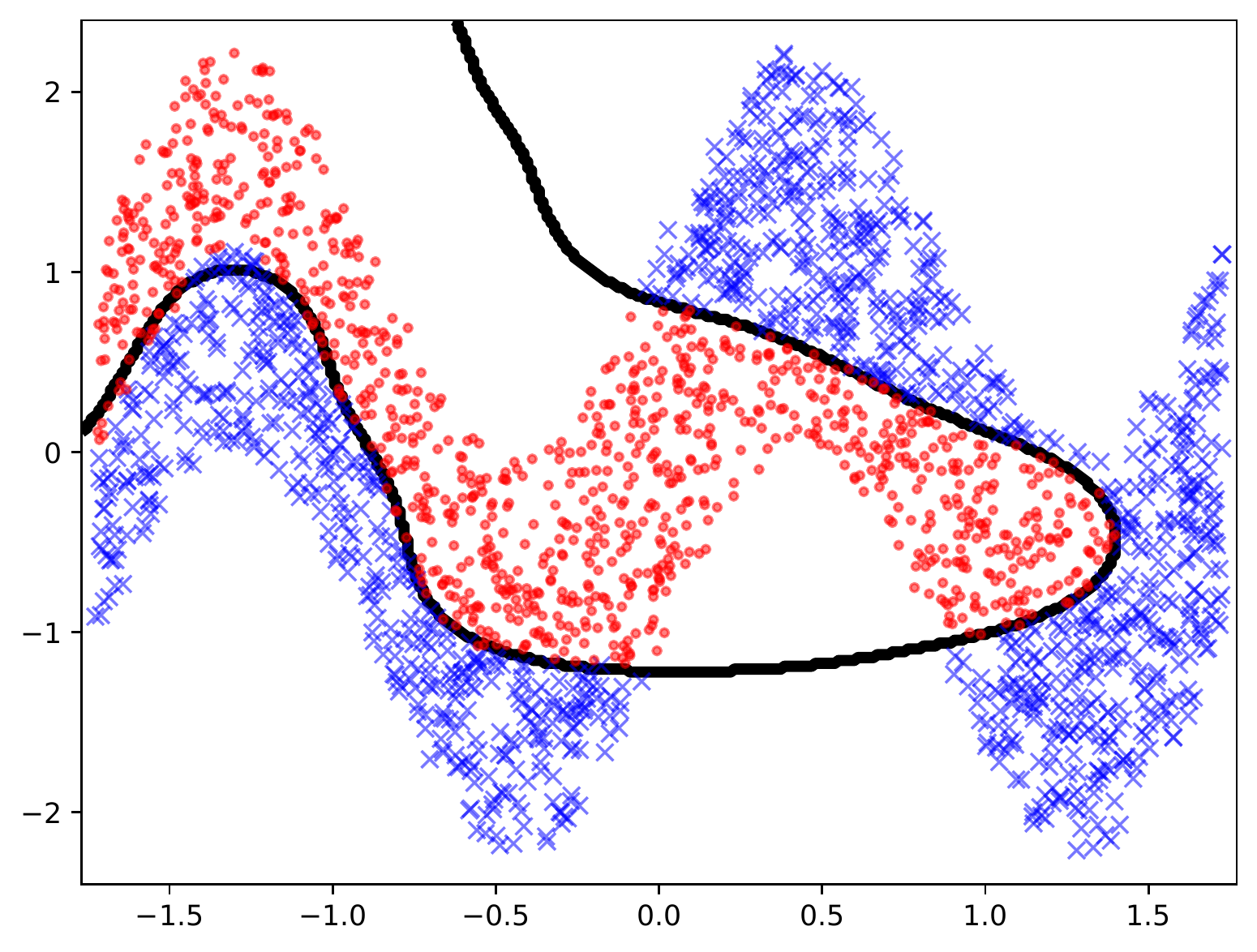}}
\subfigure[CUA]{\includegraphics[width=0.23\textwidth]{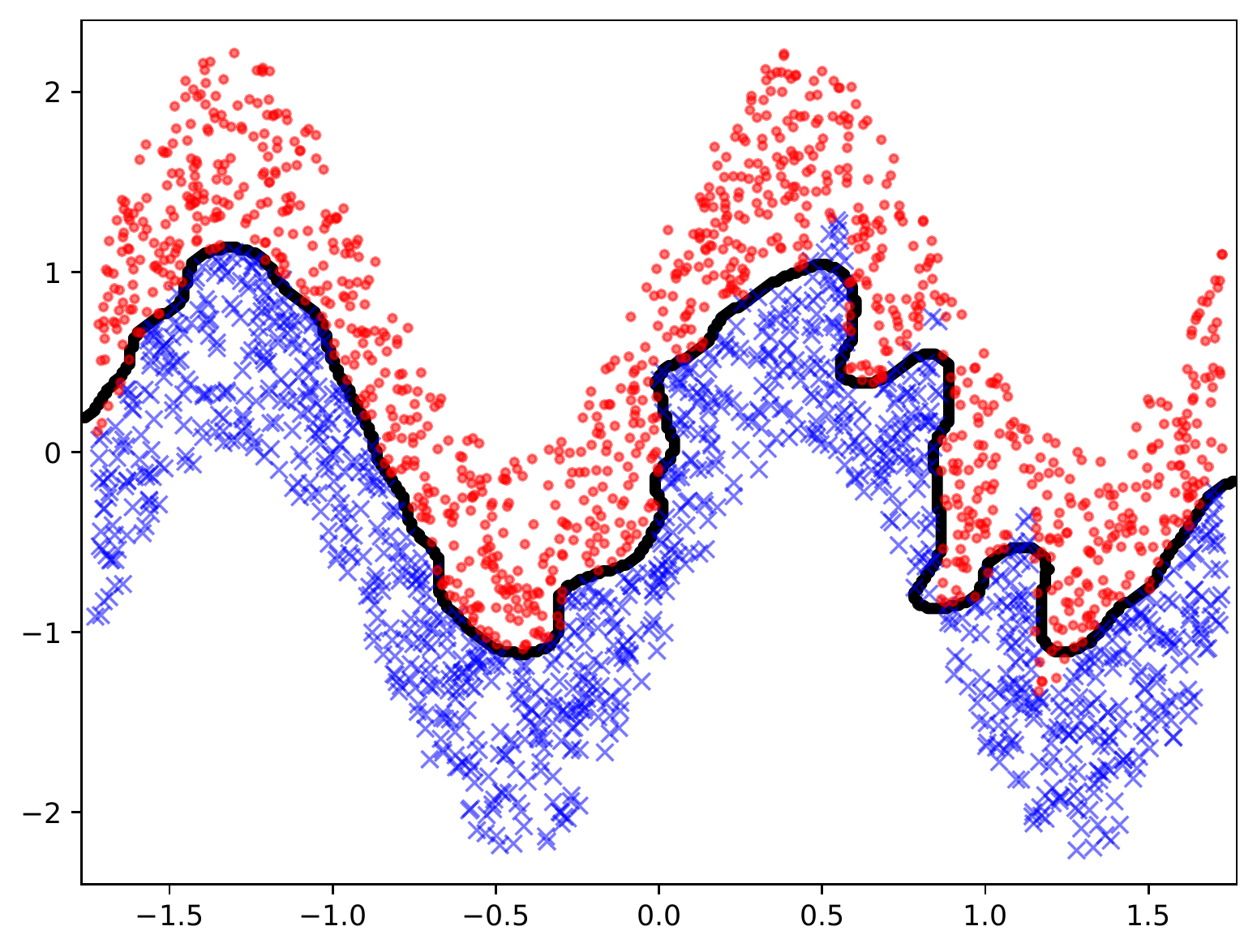}}
\subfigure[CIDA (Ours) ]{\includegraphics[width=0.23\textwidth]{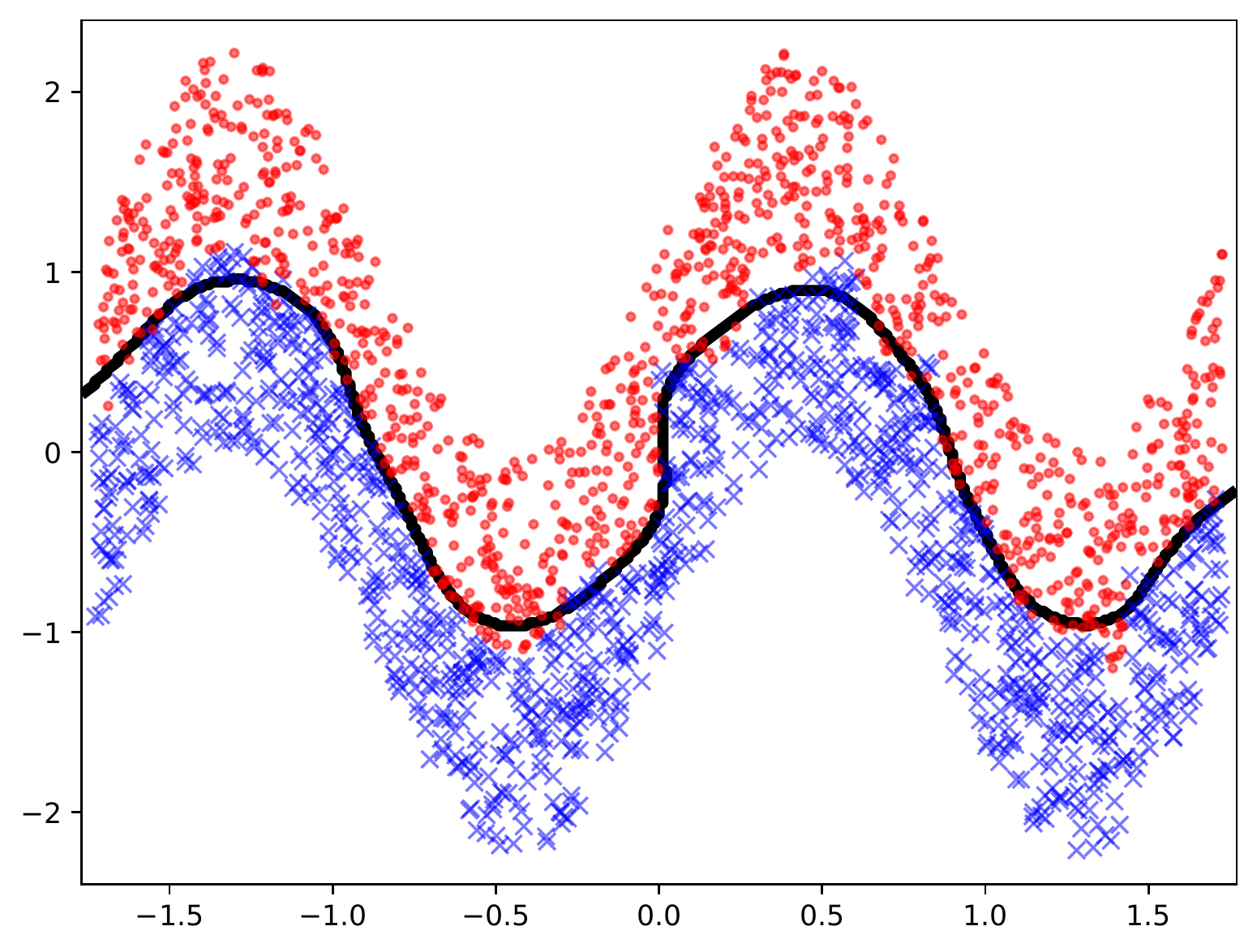}}
\vspace{-12pt}
\caption{Results on the \emph{Sine} dataset with $12$ domains. The first $5$ domains are source domains marked by green boxes. Red dots and blue crosses are positive and negative data samples. Black lines show the decision boundaries generated according to model predictions.}
\label{fig:toy-sin}
\vskip -0.5cm
\end{figure*}

\subsection{Baselines and Implementations}
We compare variants of CIDA with state-of-the-art domain adaptation methods including Domain Adversarial Neural Network (\textbf{DANN})~\cite{DANN}, Conditional Domain Adversarial Neural Network (\textbf{CDANN})~\cite{CDANN}, Adversarial Discriminative Domain Adaptation
(\textbf{ADDA})~\cite{ADDA}, Margin Disparity
Discrepancy (\textbf{MDD})~\cite{MDD}, and Continuous Unsupervised Adaptation (\textbf{CUA})~\cite{CUA}.
ADDA and MDD merge data with different domain indices into one source and one target domains; DANN, CDANN, and CUA divide the continuous domain spectrum into several separate domains and perform adaptation between multiple source and target domains. CUA adapts from the source domains to each target domain one-by-one from the closest target to the farthest one.
For a fair comparison with CIDA, all baselines use both $\x$ and the domain index $u$ as inputs to the encoder.

All methods are implemented using PyTorch~\cite{PyTorch} with the same neural network architecture. $\lambda_d$ is chosen from $\{0.2, 0.5, 1.0, 2.0, 5.0\}$ and kept the same for all tasks associated with the same dataset (see the Supplement for more details about training).

\begin{table*}[!t]
\begin{footnotesize}
\vskip -0.3cm
\begin{center}
\caption{ \textbf{\emph{Rotating MNIST} accuracy (\%) for various adaptation methods.} We report the accuracy at the source domain and each target domain. $X^\circ$ denotes the domain whose images are Rotating by $X^\circ$ to $X+45^\circ$. The last column shows the average accuracy across target domains. We use \textbf{bold face} to mark the best results.}
\label{tab:mnist}
\vspace{1mm}
\begin{tabular}{lcccccccccc}
\hline
Method & \# Target Domains & $0^\circ$ (Source) & $45^\circ$ & $90^\circ$ & $135^\circ$ & $180^\circ$ & $225^\circ$ & $270^\circ$ & $315^\circ$ & Average \\
\hline
Source-Only & - & 98.4 & 81.4 & 29.8 & 33.6 & 41.4 & 39.0 & 30.4 & 81.1 & 48.1 \\
ADDA        & 1 & 95.0 & 70.2 & 25.5 & 44.0 & 59.2 & 46.2 & 23.7 & 61.4 & 47.2 \\
DANN        & 1 & 98.1 & 80.1 & 44.8 & 42.2 & 43.6 & 46.8 & 57.3 & 79.3 & 56.3 \\
CUA         & 7 & 91.4 & 73.9 & 60.1 & 55.0 & 52.7 & 45.1 & 55.2 & 88.4 & 61.5 \\
CIDA (Ours) & $\infty$ & \textbf{99.1} & 87.2 & 56.7 & 79.6 & 91.2 & \textbf{91.5} & \textbf{96.2} & \textbf{97.5} & 85.7 \\
PCIDA (Ours) & $\infty$ & 98.6 & \textbf{90.1} & \textbf{82.2} & \textbf{90.5} & \textbf{91.9} & 87.1 & 80.0 & 88.2 & \textbf{87.1} \\
 \hline
 \vspace{-10mm}
\end{tabular}
\end{center}
\end{footnotesize}
\end{table*}

\subsection{Toy Datasets}
To gain insight into the differences between CIDA and the baselines, we start with two toy datasets: \emph{Circle} and \emph{Sine}.

\textbf{Circle Dataset} includes $30$ domains indexed from $1$ to $30$.
\figref{fig:toy-circle}(a) shows the $30$ domains in different colors. We also use arrows to indicates domain $1$ and domain $15$. Each domain contains data on a circle. The task is binary classification. \figref{fig:toy-circle}(b) shows positive samples as red dots and negative samples as blue crosses. As shown in the figure, the ground-truth decision boundary continuously evolves with the domain index. We use domains $1$ to $6$ as source domains and the rest as target domains.
\figref{fig:toy-circle} compares the results of CIDA  with the baselines. The figure shows that overall categorical DA methods perform poorly when asked to align domains with continuous indices. CUA is the best performing baseline since it incrementally adapts $24$ pairs of domains. Still, CUA's performance is inferior to CIDA which produces a more accurate decision boundary.

\textbf{Sine Dataset} includes $12$ domains as shown in \figref{fig:toy-sin}(a). Each domain covers $\frac{1}{6}$ the period of the sinosoid. We consider the first $5$ domains as source domains and the rest as target domains.
\figref{fig:toy-sin} shows the results. The figure shows that it is very challenging for the baselines to capture the ground-truth decision boundary. The baselines either produce incorrect decision boundaries (ADDA and MDD) or only capture the correct trend with very rugged boundaries (DANN, CDANN and CUA). In contrast, CIDA can successfully recover the ground-truth boundary.
We also note that while CUA  performed better than the other baselines on the Circle dataset, it performed worse than DANN and CDANN on the Sine dataset. This is because CUA performs incremental pairwise adaptation; it fails on the pair $(5,9)$, and this failure propagates to the following domains.

Overall, both the results from the Circle and Sine datasets demonstrate that CIDA captures the underlying relationship between the domain index and the classification task and leverages it to improve performance. In contrast, the baselines cannot accurately capture this relationship, and hence yield worse results.

\newcommand{\tabincell}[2]{\begin{tabular}{@{}#1@{}}#2\end{tabular}}

\begin{table*}[t]
\begin{footnotesize}
\vskip -0.3cm
\begin{center}
\caption{ \textbf{Accuracy (\%) for intra-dataset adaptation.} `\emph{SHHS}@Outside $\rightarrow$ \emph{SHHS}@(52,75]' means transferring from age range outside (52,75] (i.e., [44,52]$\cup$(75,90]) to (52,75] within \emph{SHHS}. `SO' is short for `Source-Only'. 
We use \textbf{bold face} mark the best results.
} 
\label{tab:intra}
\vspace{1mm}
\begin{tabular}{clcccccccc}
\hline
& Task  & SO  & ADDA & DANN & CDANN & MDD & CUA & CIDA & PCIDA \\
\hline
\multirow{3}{*}{\tabincell{c}{Domain\\Extrapolation}} & \emph{SHHS}@[44,52] $\rightarrow$ \emph{SHHS}@(52,90] & 77.4 & 78.0 & 77.1 & 77.5 & 77.7 & 77.4 & 79.8 & \textbf{80.6}  \\
 & \emph{MESA}@[54,58] $\rightarrow$ \emph{MESA}@(58,95]   & 80.1 & 80.7 & 79.9 & 80.4 & 80.3 & 80.1 & \textbf{82.7} & 82.5 \\
 & \emph{SOF}@[75,82] $\rightarrow$ \emph{SOF}@(82,90]   & 74.7 & 74.8 & 74.2 & 74.4 & 74.6 & 74.5 & \textbf{76.7} & \textbf{76.7} \\
\hline
\multirow{3}{*}{\tabincell{c}{Domain\\Interpolation}} & \emph{SHHS}@Outside $\rightarrow$ \emph{SHHS}@(52,75] & 82.4 & 81.7 & 82.5 & 82.3 & 82.5 & 82.4 & 82.2 & \textbf{83.7}  \\
& \emph{MESA}@Outside $\rightarrow$ \emph{MESA}@(58,75]   & 83.5 & 83.5 & 83.2 & 83.3 & 83.8 & 83.4 & 83.5 & \textbf{84.7} \\
& \emph{SOF}@Outside $\rightarrow$ \emph{SOF}@(79,86]   & 71.8 & 71.5 & 71.4 & 70.9 & 71.8 & 71.5 & 71.8 & \textbf{73.6}  \\
 \hline
\end{tabular}
\end{center}
\vskip -0.7cm
\end{footnotesize}
\end{table*}
\begin{table*}[t]
\begin{footnotesize}
\begin{center}
\caption{ \textbf{Accuracy (\%) for cross-dataset adaptation.} We use \textbf{bold face} to mark the best results.}
\label{tab:cross}
\vspace{1mm}
\begin{tabular}{lcccccccc}
\hline
Task  & Source-Only  & ADDA & DANN & CDANN & MDD & CUA & CIDA & PCIDA \\
\hline
 \emph{SOF} $\rightarrow$ \emph{SHHS}
 & 75.6 & {76.0} & 75.2 & 75.6 & 75.8 & 75.3 & 75.9 & \textbf{80.1}  \\
  \emph{SOF} $\rightarrow$ \emph{MESA}
 & 74.0 & 75.1 & 74.6 & {75.2} & 74.9 & 73.6 & 74.8 & \textbf{80.0}  \\
  \emph{SHHS} $\rightarrow$ \emph{MESA}
 & 82.8& 83.0 & 82.6 & 82.1 & 83.0 & 82.1 & {83.2} & \textbf{85.3}  \\
  \emph{MESA} $\rightarrow$ \emph{SHHS}
 & 80.7 & {81.8} & 80.9 & 80.9 & 81.2 & 81.0 & 80.8 & \textbf{83.4}  \\
 \emph{SHHS} $\rightarrow$ \emph{SOF}
 & 78.7 & 79.5 & 79.0 & 79.2 & 79.7 & 79.1 & \textbf{81.1} & {80.9}  \\
 \emph{MESA} $\rightarrow$ \emph{SOF}
 & 75.9 & 76.6 & 77.0 & 76.9 & 76.9 & 76.0 & \textbf{79.3} & {79.0}  \\
 \hline
\end{tabular}
\end{center}
\vskip -0.7cm
\end{footnotesize}
\end{table*}

\subsection{Rotating MNIST}
We further evaluate our methods on the \emph{Rotating MNIST} dataset. The goal is to adapt from regular MNIST digits with mild rotation to significantly Rotating MNIST digits. We designate images that are Rotating by $0^\circ$ to $45^\circ$ as the labeled source domain, and assign images Rotating by $45^\circ$ to $360^\circ$ to the target domains. Naturally, the domain index is the rotation angle of the image. Since the target domain has a much larger range of rotation angles, we split the target domain into seven target domains for categorical domain adaptation baselines, DANN and CUA. These seven target domains contain images Rotating by $[45,90),[90,135),\cdots,[315,360)$ degrees, respectively. \tabref{tab:mnist} compares the accuracy our proposed CIDA/PCIDA with different baselines.
We can see ADDA and DANN hardly improve and sometimes even decrease the accuracy compared to not performing adaptation at all. This is because without capturing the underlying structure, adversarial encoding alignment may harm the transferability of the data. CUA's performs fairly well in target domains near source domains but poorly in distant domains.\footnote{We note that CUA's performance on our Rotating MNIST data is worse than in the original papers, possibly because our Rotating MNIST has images rotated by all angles as opposed to only $8$ fixed angles. Also we are using different model architectures. Please refer to the Supplement for details.} On the other hand, CIDA and PCIDA can learn such domain structure and successfully adapt the knowledge from source domains to target domains (see the Supplement for more details such as model architectures).

\subsection{Healthcare Datasets}
\textbf{Dataset Description.}
We use three medical datasets, Sleep Heart Health Study (\emph{SHHS})~\cite{SHHS}, Multi-Ethnic Study of Atherosclerosis (\emph{MESA})~\cite{MESA} and Study of Osteoporotic Fractures (\emph{SOF})~\cite{SOF}. Each dataset contains full-night breathing signals of subjects and the corresponding sleep stage labels (`Awake', 'Light Sleep', `Deep Sleep', and `Rapid Eye Movement (REM)'). Breathing signals are split into 30-second segments with one label for each segment. We consider the task of sleep stage prediction, i.e., to predict the sleep stage label $y$ given a breathing segment $\x$. This is a natural task in sleep studies and can be performed in the patient home by having them wearing a breathing belt. The breathing signal can then serve to predict sleep stages and also detect apnea (temporary cessation of breathing).

The datasets also contain subjects' information such as age, which is a natural domain index $u$. \emph{SHHS}, \emph{MESA}, and \emph{SOF} include $2{,}651$, $2{,}055$, and $453$ subjects, respectively. On average, there are $1{,}000$ segments (i.e., $8.33$ hours of breathing signals) for each subject. Different datasets have different domain index distributions. For example, subjects' age range in \emph{SHHS} is $[44,90]$, while the age ranges for \emph{MESA} and \emph{SOF} are $[54,95]$ and $[72,90]$, respectively. Apparently \emph{SOF} subjects are much older. \emph{SHHS} subjects and \emph{MESA} subjects have similar age ranges but the distributions are actually different (see the histogram plot in the Supplement).

\textbf{Intra-Dataset Adaptation.} We first evaluate our methods' performance on adaptation across continuously indexed domains within the same dataset using `age' as the domain index. We cover two cases:
\begin{Itemize}
\item \textbf{Domain Extrapolation.} For example, the source domain has data with a domain index (age) from the range [44,52], while the target domain contains data with a domain index range of (52,90]. 
\item \textbf{Domain Interpolation.} For example, in the source domain, the domain index range is [44,52]$\cup$(75,90], while in the target domain, the domain index range is (52,75]. 
\end{Itemize}

The first three rows of \tabref{tab:intra} show the results for domain extrapolation. One observation is that directly using categorical domain adaptation only achieves minimal performance boost compared to models trained only on the source domains (Source-Only). Some methods such as DANN and CUA achieve no or even negative performance improvement. On the other hand, CIDA variants can successfully transfer across subjects with different ages and significantly improve upon all baselines. Similarly, the last three rows in \tabref{tab:intra} show the results for domain interpolation. Note that Source-Only can already achieve satisfactory accuracy in domain interpolation, since the model naturally learns the average of data from both sides (e.g., [44,52]$\cup$(75,90]) and performs prediction for the data in the middle (e.g., (52,75]). For example, in the task `\emph{SHHS}$@$Outside $\rightarrow$ \emph{SHHS}@(52,75]', Source-Only already has a high accuracy of $82.4\%$, leaving little room for improvement. Interestingly, PCIDA can still further improve the accuracy by a tangible margin. This also shows PCIDA's ability to avoid bad equilibriums by using a discriminator that predicts distributions rather than values.

\begin{table}[t]
\begin{footnotesize}
\vskip -2mm
\begin{center}
\caption{\textbf{Accuracy (\%) for the multi-dimensional domain index setting.} The task is \emph{SHHS}@[44,52] $\rightarrow$ \emph{SHHS}@(52,90].}
\label{tab:multi}
\vskip 1mm
\begin{tabular}{cccc}
\hline
\# Dimensions  & Source-Only  & CIDA & PCIDA \\
\hline
 1 & 77.4 & 79.8 & \textbf{80.6} \\
 2 & 77.6 & 81.0 & \textbf{81.1} \\
 4 & 77.7 & 81.2 & \textbf{81.3} \\
11 & 77.7 & \textbf{82.6} & \textbf{82.6} \\
\hline
\end{tabular}
\end{center}
\vskip -6mm
\end{footnotesize}
\end{table}

\textbf{Cross-Dataset Adaptation.}
Most clinical trials collect data from a population with a specific medical condition, and exclude people who have other conditions. However in practice many patients have multiple medical conditions and hence doctors need to apply the results of a particular study outside the population for which the data is collected.  Thus, in this section we consider cross-dataset adaption. Specifically, we evaluate how different methods perform when transferring among the datasets \emph{SHHS}, \emph{MESA}, and \emph{SOF}. \tabref{tab:cross} shows the accuracy of all methods in these cross-dataset settings. We observe that categorical domain adaptation barely improves upon models trained with only source domains, while CIDA and PCIDA can naturally transfer across continuously indexed domains even in the cross-dataset setting with significant performance improvement. Interestingly, when the task is hard such as transferring from \emph{SOF}, a very old people dataset, to datasets with much more diverse age range, PCIDA becomes a clear winner. But when the task is relatively easier, such as transferring from datasets with diverse age range to a very old dataset, CIDA is marginally better than PCIDA; but, this latter difference in performance is minor, and PCIDA performs well across all scenarios.

We also note that \emph{SHHS} and \emph{MESA} are both diverse datasets with similar age distribution, which is why the Source-Only model already achieves high accuracy. Interestingly, even in such cases PCIDA can still achieve stable performance gain compared to all baselines.

\textbf{Multi-Dimensional Indices.} As mentioned in \secref{sec:method}, both CIDA and PCIDA naturally generalize to multi-dimensional domain indices. To demonstrate this feature, we leverage that the \emph{SHHS} dataset includes multiple variables per patient in addition to their age, such as their heart rate, their physical and emotional health scores, etc. We combine such variables with the person's age to create a multi-dimensional domain index. (see more details on different domain indices in the Supplement). \tabref{tab:multi} shows the accuracy for multi-dimensional CIDA/PCIDA. For reference, we report corresponding accuracy for Source-Only. Note that Source-Only takes both the breathing signals ($\x$) and the domain index ($u$) as input, as is done in all previous experiments. As expected we can observe substantial improvement in accuracy with multi-dimensional domain indices. 

Note that in the case of multi-dimensional indices, \emph{domain extrapolation} means that the target domain indices are \emph{outside} the convex hull of the source domain indices. Similarly, \emph{domain interpolation} means that the target domain indices are \emph{inside} the convex hull of the source domain indices. 
\section{Conclusion}
We identify the problem of adaptation across continuously indexed domains, propose a series of methods for addressing it, and provide supporting theoretical analysis and empirical results using both synthetic and real-world medical data. Our work demonstrates the viability of efficient adaptation across continuously indexed domains and its potential impact on important real-world applications. Future work could investigate the possibility of matching higher or even infinite order moments, and the application of the proposed methods to other datasets in robotics or the medical field. 

\section*{Acknowledgement}
We thank Guang-He Lee and Mingmin Zhao for the insightful and tremendously helpful discussions. We are also grateful to Xingjian Shi, Xiaomeng Li, Hongzi Mao as well as other members at NETMIT and CSAIL for their comments to improve this paper. We would also like to thank Daniel R. Mobley and NSRR for their help with the datasets.

\bibliography{paper}
\bibliographystyle{icml2020}

\end{document}


\language0
\lefthyphenmin=2
\righthyphenmin=3

\maketitle



\section{Proof}

\begin{lemma}[\textbf{Uniqueness of Constant Expectation}]
$\z$ and $u$ are random variables.  If $\EB_{u\sim p(u|\z)}[u]$ is constant w.r.t $\z$ , then $\EB_{u\sim p(u|\z)}[u] = \EB_{u\sim p(u)}[u],\forall \z$.
\end{lemma}
\begin{proof}
Let $\EB_{u\sim p(u|\z)}[u]=\mu, \forall~\z$. We then have
$
\EB_{u\sim p(u)}[u]=\EB_{\z\sim p(\z)}\EB_{u\sim p(u|\z)}[u]=\EB_{\z\sim p(\z)} \mu = \mu
$.
\end{proof}

\begin{lemma}[\textbf{Uniqueness of Constant Expectation and Variance}]
$\z$ and $u$ are random variables.  If $\EB_{u\sim p(u|\z)}[u]$ and $\VB_{u\sim p(u|\z)}[u]$ are constants w.r.t $\z$ , then $\EB_{u\sim p(u|\z)}[u] = \EB_{u\sim p(u)}[u]$ and $\VB_{u\sim p(u|\z)}[u] = \VB_{u\sim p(u)}[u]$ for any $\z$.
\end{lemma}
\begin{proof}
Let $\EB_{u\sim p(u|\z)}[u] = \mu$ and $\VB_{u\sim p(u|\z)}[u] =\sigma^2$ for any $\z$. By the previous lemma, we have $\EB_{u\sim p(u)}[u] = \mu$. For the variance, we have: 
\begin{align*}
\VB_{u\sim p(u)}[u] &= \EB_{u\sim p(u)}[(u - \EB[u])^2] = \EB_{\z \sim p(\z)}\EB_{u\sim p(u|\z)}[(u - \EB[u|\z])^2] \\
&=\EB_{\z\sim p(\z)}\VB_{u\sim p(u|\z)}[u] 
=\EB_{\z\sim p(\z)}\sigma^2 = \sigma^2,
\end{align*}
concluding the proof.
\end{proof}

\begin{lemma}[\textbf{Optimal Discriminator for PCIDA}]\label{lem:opt_dis_pcida} With E fixed, the optimal D is
\begin{align*}
D^*_{{\mu},E}(\z) &= \mathbb{E}_{u\sim p(u | \z)} [u],\\
D^*_{{\sigma^2},E}(\z) &= \mathbb{V}_{u\sim p(u | \z)} [u],
\end{align*}
where $\z=E(\x,u)$.
\end{lemma}
\begin{proof}
The optimal $D$:
\begin{align*}
D^*_{E}(\x, u) = \argmin\limits_{D} \mathbb{E}_{(\z,u)\sim p(\z,u)} [L_d(D(\z),u)],
\end{align*}
where the objective function expands to
\begin{align*}
&\mathbb{E}_{(\z,u)\sim p(\z,u)}[L_d((D_{\mu}(\z), D_{\sigma^2}(\z)),u)]\\
= & \mathbb{E}_{(\z,u)\sim p(\z,u)}[\frac{(D_{\mu}(\z) - u)^2}{2 D_{\sigma^2}(\z)} + \frac{1}{2}\log D_{\sigma^2}(\z)]\\
= & \mathbb{E}_{\z\sim p(\z)} \mathbb{E}_{u\sim p(u | \z)}[\frac{(D_{\mu}(\z) - u)^2}{2 D_{\sigma^2}(\z)} + \frac{1}{2}\log D_{\sigma^2}(\z)].
\end{align*}
Notice that
\begingroup\makeatletter\def\f@size{9}\check@mathfonts
\begin{align*}
& \mathbb{E}_{u\sim p(u | \z)}[\frac{(D_{\mu}(\z) - u)^2}{2 D_{\sigma^2}(\z)} + \frac{1}{2}\log D_{\sigma^2}(\z)] \nonumber\\
= & \frac{\mathbb{E}[u^2|\z]}{2D_{\sigma^2}(\z)} 
- \frac{D_{\mu}(\z)}{D_{\sigma^2}(\z)}\mathbb{E}[u|\z] 
+ \frac{D_{\mu}(\z)^2}{2D_{\sigma^2}(\z)} + \frac{1}{2}\log D_{\sigma^2}(\z).
\end{align*}
\endgroup
Taking the derivative w.r.t. $D(\z)$ and setting it to $0$, we get the optimal
$D^*_{{\mu},E}(\z) = \EB[u|\z]$ and $
D^*_{{\sigma^2},E}(\z) = \VB[u|\z]$, completing the proof.
\end{proof} 

\section{Experiments}

In this section we provide more details for our experiments. The code is available at \url{https://github.com/hehaodele/CIDA}.

\subsection{Experiment on the Healthcare Datasets}

\textbf{Dataset details.} The three real-world medical datasets~\cite{SHHS,MESA,SOF} with detailed information are publicly available\footnote{\url{https://sleepdata.org/}}. They can all be freely accessed upon request and submission of relevant IRB documents. 
In \figref{fig:age-hist} we plot the histograms subjects' age in the three medical datasets. All the three datasets contains many health retaled variables of the subjects. In \tabref{tab:domain_index}, we list all the variables we considered as the domain indices.

\begin{table}[!tb]
\centering
\vskip -0.5cm
\caption{$11$ domain indices in the \emph{SHHS} dataset.}\label{tab:domain_index}
\vskip 0.09cm
\begin{tabular}{c|l} \hline
$u_1$   & Age  \\ \hline
$u_2$   & Resting heart rate  \\ \hline
$u_3$   & Gender  \\ \hline
$u_4$   & Physical functioning  \\ \hline
$u_5$   & Role limitation due to physical health  \\ \hline
$u_6$  & General health  \\ \hline
$u_7$  & Role limitation due to emotional problems  \\ \hline
$u_8$   & Energy/fatigue  \\ \hline
$u_9$   & Emotional well being  \\ \hline
$u_{10}$   & Social functioning  \\ \hline
$u_{11}$   & Pain Level \\ \hline
\end{tabular}
\vskip 0.2cm
\end{table}

\textbf{Implementations.}
We use the same neural network architecture in all methods for fair comparison. \tabref{table:structure} shows the neural network architecture for the encoders taking breathing signals $\x$ as input. `Number' in the tables indicates the number of corresponding blocks stacked in the network. The predictor includes $3$ fully connected layers, each with batch normalization and ReLU. Similarly, the discriminator includes $5$ fully connected layers. For the baseline models, we explore different $\lambda_d$ (the hyperparameter for the discriminator term) in the range $\{0.2, 0.5, 1.0, 2.0, 5.0\}$ and find that $\lambda_d = 2.0$ produce stable and the best results in the toy datasets. We follow recommendations from the original papers for other hyperparameters. We set $\lambda_d=2.0$ for all methods including CIDA/PCIDA. We train all models using the Adam optimizer~\cite{Adam} with a learning rate of $10^{-4}$. We run all experiments on a server with four NVIDIA Titan Xp GPUs.

\begin{table*}[!tb]
\begin{footnotesize}
\centering
\vskip -0.1cm
\caption{Network structure for the encoder.}\label{table:structure}
\vskip 0.05cm
\begin{tabular}{ccccccc} \hline
Kernel & Stride & Channel In & Channel Middle & Channel Out & Type & Number \\ \hline
13 & 2 & 1 & -   & 64 & Conv & 1 \\
9 & 1 & 64       & 64  & 64 & ResBlock & 1 \\
9 & 2 & 64       & 64  & 128 & ResBlock & 1 \\
9 & 1 & 128      & 128 & 128 & ResBlock & 1 \\
7 & 1 & 128      & 128 & 256 & ResBlock & 1  \\
9 & 5 & 256      & 256 & 256 & ResBlock & 1  \\
5 & 1 & 256      & 256   & 512 & ResBlock & 1  \\
5 & 1 & 512      & 512   & 512 & ResBlock & 1  \\
5 & 1 & 512      & 384   & 384 & ResBlock & 1  \\
9 & 5 & 384      & 384   & 384 & ResBlock & 1  \\
3 & 1 & 384      & 384   & 384 & ResBlock & 1  \\
5 & 1 & 384      & 384   & 384 & ResBlock & 1  \\ \hline
\end{tabular}
\vskip -0.0cm
\end{footnotesize}
\end{table*}

\begin{figure*}[!tb]
\centering     
\subfigure[SHHS]{\includegraphics[width=0.32\textwidth]{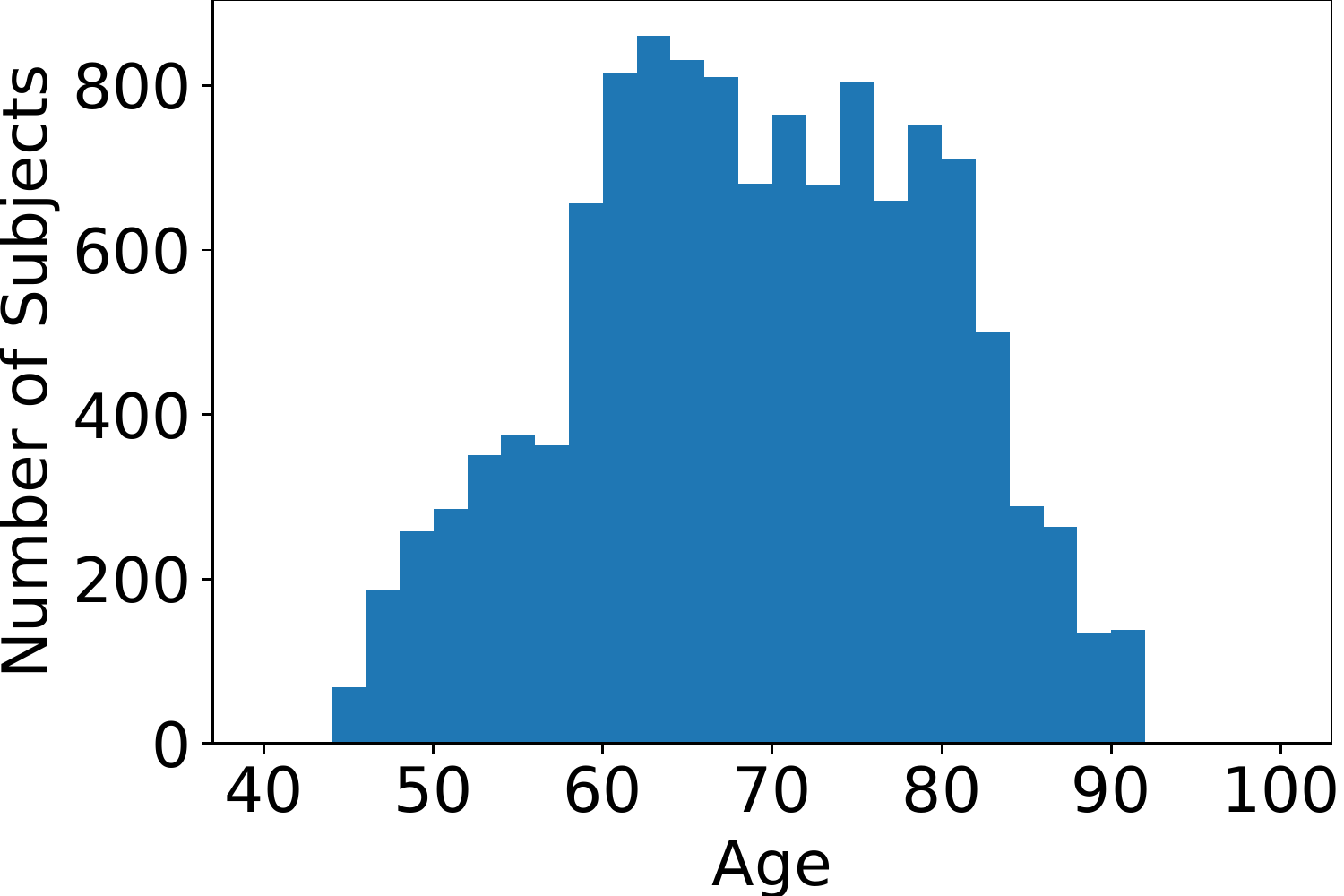}}
\subfigure[MESA]{\includegraphics[width=0.32\textwidth]{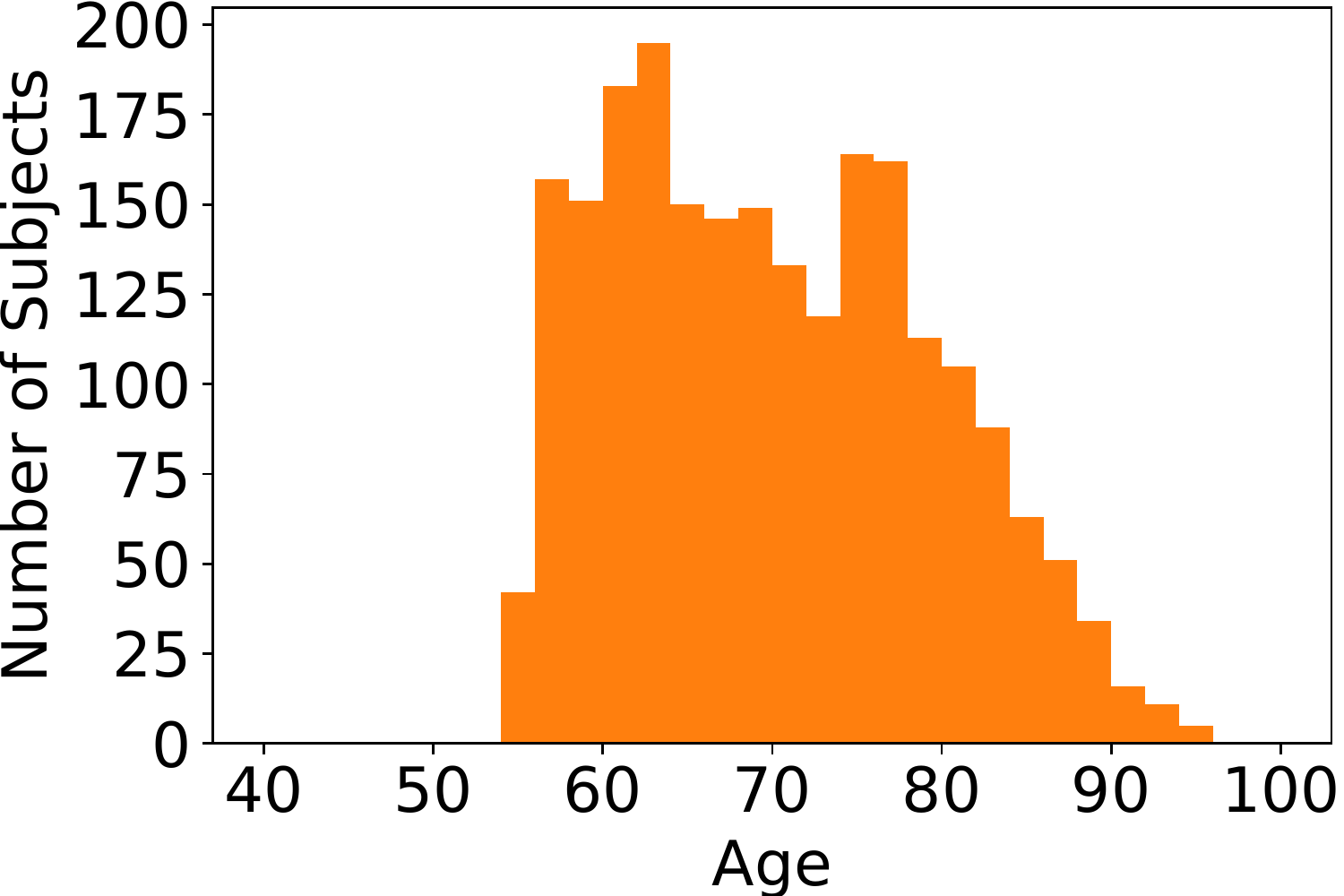}}
\subfigure[SOF]{\includegraphics[width=0.32\textwidth]{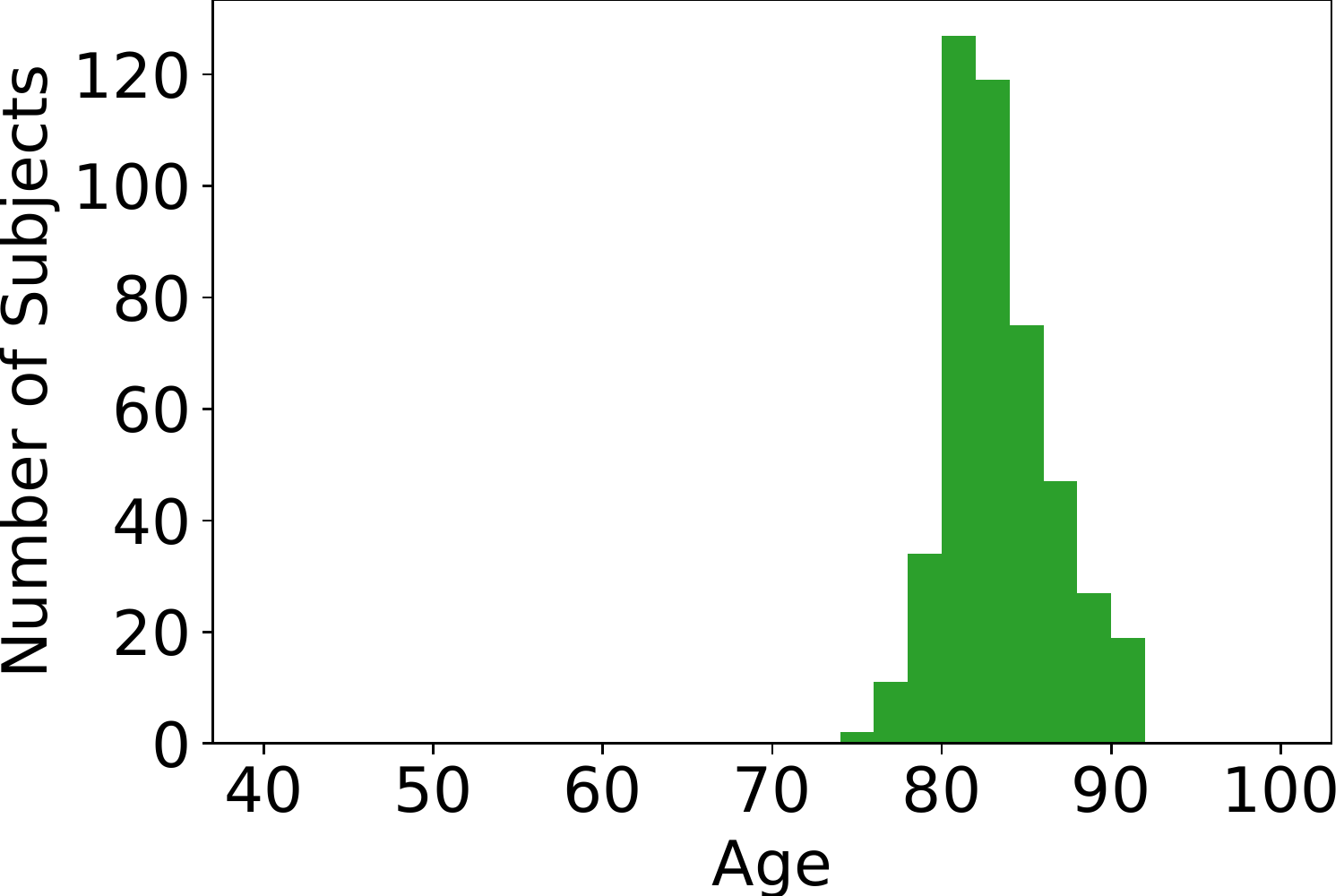}}

\caption{Age histograms for three medical datasets.}
\label{fig:age-hist}
\end{figure*}

\subsection{Experiment on the Rotating MNIST}

\textbf{Dataset details.} 
In our \emph{Rotating MNIST}, there are $60{,}000$ images in each domain index interval spanning $45$ degrees from $[0^\circ, 45^\circ)$ to $[315^\circ,360^\circ)$. They are generated by rotating each of the $60{,}000$ image in the MNIST training set by a angle randomly sampled the corresponding domain index interval. Therefore, this dataset contains images with rotation angles evenly spread in the range $[0, 360^\circ)$. We note that this is different from the \emph{Rotating MNIST} dataset in \cite{CUA}, where the images are Rotating by $8$ fixed angles. Another difference is that in our Rotating MNIST, the amount of data in target domains is 7 times as many as that in source domain while in \cite{CUA}, the target domain has the same amount of data as the source domain.

\textbf{Implementations.}
We use the same neural network architecture in all methods for fair comparison. Mainly, we use a four-layer convolutional neural networks to encode the image and a three-layer MLP to make the prediction, while the discriminator is a four-layer MLP. In addition, we make two augmentations to provide the model with a stronger inductive bias. First, we add a Spacial Transfer Network (STN)~\cite{STN} to the image encoder. Basically, the STN will take the image and the domain index as input and output a set of rotation parameters which are then applied to rotate the given image. Second, we add the dropout layers to the STN and the ConvNet backbone. As mentioned in \cite{gal2016dropout}, dropout can be viewed as a way of performing Bayesian inference. Here, we use this dropout switch to make image encoder either deterministic or probabilistic. For more details, please refer to our code. 

\subsection{Experiments on the Toy Datasets}

\textbf{Visualization of the decision boundary (approximately).}
Unlike shallow models such as logistic regression, plotting deep neural networks' exact decision boundaries is not straightforward. To generate a virtual decision boundary for visualization, we fit an SVM with the RBF kernel by neural networks' prediction and draw the decision boundary of the SVM. To be fair, when fitting the SVM, we ensure that the fitting accuracy is the same for all deep learning models. Note that since the generated boundaries are not exact, we can observe some data points on the wrong side of the boundaries.
\section{Discussion}

\subsection{Categorical Domains versus Continuously Indexed Domains}
\textbf{Continuous Indices.} As mentioned in the main paper, the hypothesis of `continuous indices' is that input $\x$ and labels $y$ are drawn from $p(\x, y | u)$ given a specific domain index $u\in \mathcal{U}$, and that $p(\x, y | u)$ (and $p(y | \x, u)$) is continuous w.r.t. $u$. Therefore, CIDA tries to produce correct predictions in a continuous range of target domains by effectively capturing the underlying relation (functional) between $p(y | \x, u)$ and $u$. 

\textbf{Distance Metrics.} Such a hypothesis comes with a distance metric for domain indices, which are captured by the regression loss (e.g., euclidean distance for $L_2$ loss) in the discriminator. This is a key difference between CIDA and categorical domain adaptation, where any pair of domains effectively has the same distance. This is also true for categorical domain adaptation methods such as \cite{LSGAN}. Note that \cite{LSGAN} uses a least-square loss as a surrogate for cross-entropy to perform domain \emph{classification} in the discriminator, therefore still treating different domains as equal. This is substantially different from CIDA where the $L_2$ loss and the Gaussian (or Gaussian Mixture Model) loss are use to regress the domain indices.

\subsection{Matching $p(u | \z)$ versus Matching $p(\z | u)$} 
In general, matching the entire $p(u | \z)$ for any $\z$ is equivalent to matching the entire $p(\z | u)$ for any $u$. This is because $p(u | \z) = p(u) \iff p(\z | u) = p(\z) \iff u \indep \z$. However, matching the mean and variance of $p(u | \z)$ for any $\z$ is \textbf{different} from matching the mean and variance of $p(\z | u)$ for any $u$. Considering the dimension of $\z$ is much higher than that of $u$, the former is actually \textbf{stronger} alignment.

To see this, consider a simplified case where $\z \in \{1,2,3,4\}^{100}$ and $u \in  \{1,2,3,4\}$. Matching the mean and variance of $p(u | \z)$ requires matching the mean and variance of $4^{100}$ univariate distributions, i.e., $2\times 4^{100}$ parameters in total. On the other hand, Matching the mean and variance of $p(\z | u)$ only requires matching the mean and variance of $4$ $100$-dimensional distributions, i.e., $400$ parameters in total. Therefore the former implies stronger alignment.

\bibliography{paper}
\bibliographystyle{plain}

